\documentclass{ecai} 

\usepackage{microtype}
\usepackage{graphicx}
\usepackage{subcaption}
\usepackage{latexsym}
\usepackage{amssymb}
\usepackage{amsmath}
\usepackage{mathtools}
\usepackage{amsthm}
\usepackage{booktabs}
\usepackage{bm}
\usepackage{makecell}
\usepackage{enumitem}
\usepackage{graphicx}
\usepackage{color}
\usepackage{tikz}
\usepackage{multirow}
\usetikzlibrary{arrows,matrix, positioning,  shapes.geometric}
\usepackage{algorithm, algorithmic}
\usepackage{listofitems}
\usepackage{float} 

\usepackage{amsmath,amsfonts,bm}

\def\eqref#1{equation~\ref{#1}}
\def\Eqref#1{Equation~\ref{#1}}

\def\1{\bm{1}}

\DeclareMathAlphabet{\mathsfit}{\encodingdefault}{\sfdefault}{m}{sl}
\SetMathAlphabet{\mathsfit}{bold}{\encodingdefault}{\sfdefault}{bx}{n}

\DeclareMathOperator*{\argmin}{arg\,min}

\newcommand\m[1]{\ensuremath{\mathcal{#1}}}
\newcommand{\feature}{\bm{\psi}}
\newcommand{\cost}{\bm{y}}
\newcommand{\predcost}{\bm{\hat{y}}}
\newcommand{\decisionvar}{\bm{w}}
\newcommand{\solution}{\decisionvar^\star}

\newcommand{\ML}{\m{M}_{\theta}}
\newcommand{\MLnoomega}{\m{M}}
\newcommand{\proxy}{\m{G}_{\theta}}

\newcommand\regret{\m{{\mathit{Regret}}}}
\newcommand{\feas}{\m{F}}

\newcommand{\spoplus}{\m{L}_{\mathit{SPO}^{+} }} 
\newcommand{\SCE}{  \m{L}_{SCE}  } 

\newcommand{\dys}{\mathit{DYS}}
\newcommand{\smooth}{\mathit{SmOPT}}

\newtheorem{theorem}{Theorem}

\newtheorem{corollary}[theorem]{Corollary}
\newtheorem{proposition}{Proposition}

\definecolor{Green}{cmyk}{1.0,0.,0.9975,0}
\definecolor{olive}{rgb}{0.73, 0.72, 0.42}

\newcommand{\BibTeX}{B\kern-.05em{\sc i\kern-.025em b}\kern-.08em\TeX}

\tikzstyle{mynode}=[thick,inner sep={.75\pgflinewidth}, draw=black,fill=white,circle,minimum size=3]
\newsavebox{\mybox}
\sbox{\mybox}{%
\begin{tikzpicture}[x=0.35cm,y=0.5cm]
  \readlist\Nnod{3,4,4,1} 
  \foreachitem \N \in \Nnod{ 
    \foreach \i [evaluate={\x=\Ncnt; \y=\N/2-\i+0.5; \prev=int(\Ncnt-1);}] in {1,...,\N}{ 
      \node[mynode] (N\Ncnt-\i) at (\x,\y) {};
      \ifnum\Ncnt>1 
        \foreach \j in {1,...,\Nnod[\prev]}{ 
          \draw[thick] (N\prev-\j) -- (N\Ncnt-\i); 
        }
      \fi 
    }
  }
\end{tikzpicture}
}

\begin{document}


\begin{frontmatter}


\paperid{0517} 


\title{Minimizing Surrogate Losses for Decision-Focused Learning using Differentiable Optimization}

\author[A]{\fnms{Jayanta}~\snm{Mandi} \thanks{Corresponding Author. Email: jayanta.mandi@kuleuven.be}}
\author[A]{\fnms{Ali İrfan}~\snm{Mahmutoğulları}}
\author[A]{\fnms{Senne}~\snm{Berden}} 
\author[A]{\fnms{Tias}~\snm{Guns}}
\address[A]{ KU Leuven, Department of Computer Science, Belgium}

\begin{abstract}
Decision-focused learning (DFL) trains a machine learning (ML) model to predict parameters of an optimization problem, to directly minimize decision regret, i.e., maximize decision quality. Gradient-based DFL requires computing the derivative of the solution to the optimization problem with respect to the predicted parameters. However, for many optimization problems, such as linear programs (LPs), the gradient of the regret with respect to the predicted parameters is zero almost everywhere. Existing gradient-based DFL approaches for LPs try to circumvent this issue in one of two ways: (a) smoothing the LP into a differentiable optimization problem by adding a quadratic regularizer and then minimizing the regret directly or (b) minimizing surrogate losses that have informative (sub)gradients. In this paper, we show that the former approach still results in zero gradients, because even after smoothing the regret remains constant across large regions of the parameter space. To address this, we propose minimizing surrogate losses, even when a differentiable optimization layer is used and regret can be minimized directly. Our experiments demonstrate that minimizing surrogate losses allows differentiable optimization layers to achieve regret comparable to or better than surrogate-loss based DFL methods.
Further, we demonstrate that this also holds for DYS-Net, a recently proposed differentiable optimization technique for LPs, that computes \emph{approximate} solutions and gradients through operations that can be performed using feedforward neural network layers. Because DYS-Net executes the forward and the backward pass very efficiently, by minimizing surrogate losses using DYS-Net, we are able to attain regret on par with the state-of-the-art while reducing training time by a significant margin.
\end{abstract}
\end{frontmatter}
\section{Introduction}
Many real-world decision-making problems can be cast as combinatorial optimization problems.
Some parameters of these optimization problems (e.g., production costs or travel times) are typically unknown due to uncertainty when the decisions are made.
In the context of data-driven contextual optimization~\citep{contextual2025},
the problems of predicting parameters of an optimization problem can be viewed as ``predict-then-optimize'' (PtO) problems, including two key steps -- the \textit{prediction} of the unknown parameters given the observed contextual information and the subsequent \textit{optimization} using those predicted parameters.
Decision-focused learning (DFL) \citep{mandi2023decision} trains machine learning (ML) models to predict uncertain parameters by \emph{directly} minimizing the decision loss, a measure of the quality of the decisions made using the predicted parameters.

Among ML models, neural networks have emerged as highly successful due to their ability to learn complex patterns from data \citep{alzubaidi2021review}. However, training neural networks relies on gradient-based learning, which poses a fundamental challenge for DFL.
Gradient-based DFL methods entail computing the partial derivatives of the optimization problem with respect to the predicted parameters. However, for combinatorial optimization problems, these partial derivatives are zero almost everywhere, since slight parameter changes rarely alter the solution, except at transition points, where the solution changes abruptly, and the partial derivatives do not exist.
%
In this paper we consider combinatorial optimization problems which can be formulated as integer linear programs (ILPs) or mixed integer linear programs (MILPs).
To obtain informative gradients for DFL on (MI)LPs, previous works employ two broad categories of approaches: (a) smoothing the (MI)LP into a differentiable optimization problem~\citep{aaai/WilderDT19, MandiNEURIPS2020, Mipaal}, and then minimizing the decision loss by differentiating through the `smoothed' problem, or (b) using surrogate loss functions \citep{elmachtoub2022smart, mulamba2020discrete, mandi22a}, for which informative gradients or subgradients exist.

\paragraph{Contribution.}
In this paper, we focus on the first category of DFL approaches. 
In this category, the (MI)LP is turned into a differentiable optimization problem as follows: the linear programming (LP) relaxation of the (MI)LP is first considered by removing the integrality constraints on the variables. 
Then the LP is converted into `smooth' quadratic program (QP) by adding a quadratic regularizer to the objective function.
In practice, the weight on the regularizer is kept low so that the smoothing strength remains small and does not overshadow the original objective.
As mentioned earlier, approaches in this category minimize the decision loss on the training dataset by smoothing the LP into a QP and then computing the Jacobian of its solution with respect to the predicted parameters using a differentiable solver (e.g., \textsl{Cvxpylayers} \citep{agrawal2019differentiable}). The motivation for minimizing decision loss on the training set follows the principle of empirical risk minimization (ERM) \citep{vapnik1991}, which suggests that improvements on the training data should lead to lower decision loss on unseen test instances.

Our main argument is that direct minimization of decision loss using differentiable optimization is \emph{ineffective} for gradient-based DFL. 
This is because QP smoothing makes the optimization problem differentiable by smoothing out abrupt transitions in the solution. However, this smoothing often leads to constant solutions across large regions of the parameter space. As a result, the gradient of the decision loss with respect to the predicted parameters becomes zero during backpropagation, rendering it ineffective for gradient-based DFL.

To overcome this challenge, we propose minimizing surrogate losses instead, even when direct minimization of the decision loss is feasible using a differentiable optimization layer.
At first glance, this might seem counterintuitive: \textit{Why would one optimize a surrogate loss when the decision loss can be directly minimized using a differentiable optimization solver?} We answer this question by demonstrating that the surrogate loss provides gradients that are useful for gradient-based training.
We empirically demonstrate that minimizing the surrogate losses using differentiable optimization solvers provides gradients that lead to lower decision loss on unseen instances.

Furthermore, training in DFL is computation-intensive. 
Both types of approaches require solving the (smoothed) (MI)LP for each training instance with the predicted parameters in each epoch to compute the decision loss. This poses significant scalability challenges, and increases training time.
To ease the computational burden, 
\citet{mckenzie2024learning} recently developed a fully-neural optimization layer, DYS-Net.
Unlike Cvxpylayers, DYS-Net computes the solution to the quadratically regularized LP and the gradient of the solution using neural operations such as matrix vector multiplications. 
By doing so, DYS-Net offers the potential to significantly speed up DFL training.
However, this potential has not been fully established because
\citet{mckenzie2024learning} use datasets where the true parameters are \textit{not} observed, while most DFL benchmarks~\citep{pyepo} assume otherwise. So, DYS-Net has not been compared to the state-of-the-art DFL methods in this setting.
We empirically show that our proposal of minimizing surrogate losses with DYS-Net yields test-time regret comparable to state-of-the-art DFL methods,  which \emph{direct regret minimization cannot match due to the zero-gradient issue.}
In summary, this paper makes the following contributions:
\begin{itemize}
    \item We demonstrate that although `smoothing' turns the (MI)LP into a differentiable optimization problem, the gradient of the empirical decision loss remains zero over most of the parameter space.
    \item We empirically show that for quadratic smoothing, minimizing surrogate losses leads to lower decision loss on test data than minimizing the decision loss. 


    \item 
    By minimizing the surrogate loss, using DYS-Net as the differentiable solver, we achieve decision loss comparable to state-of-the-art methods across three types of optimization problems, while significantly reducing training time (often by a factor of three).
\end{itemize}

\section{Background}
This work focuses on LPs, ILPs and MILPs.
An LP can be represented in the following form:
\begin{equation}
\label{eq:big_LP}
\min_{\decisionvar} \cost^\top \decisionvar \; \; \text{s.t.} \;\;A \decisionvar = \mathbf{b}; \;\;C\decisionvar \leq \mathbf{d}; 
\end{equation}
\noindent where $\decisionvar \in  \mathbb{R}^K$ is a decision variable.
The optimal solution for a given cost parameter, $\cost \in  \mathbb{R}^K$, is denoted by
$\solution (\cost)$. 
Note that any constraints in the form of $C \decisionvar \leq \mathbf{d} $ can be
converted to equality by introducing slack variables and the LP can be transformed in the following standard form:
\begin{equation}
\label{eq:standard_LP}
\min_{\decisionvar} \cost^\top \decisionvar \; \; \text{s.t.} \;\;A \decisionvar = \mathbf{b}; \;
    \decisionvar \geq \bm{0}
\end{equation}
For brevity, we use $\feas$ to denote the feasible space. So, for the standard LP formulation, 
$\feas = \{ \decisionvar \in \mathbb{R}^K \mid A \decisionvar = \mathbf{b}  \; ;
    \decisionvar \geq \bm{0} \}$.
    Unless it is explicitly stated otherwise, $\solution$ will denote  $ \solution (\cost)$.
ILPs differ from LPs in that the all decision variables $\decisionvar$ are restricted to integer values. MILPs generalize ILPs by allowing only a subset of variables to be integer, while the rest can be continuous. Hence, MILPs can be seen as a superset that includes both LPs and ILPs. Although we describe our approach for MILPs, it naturally applies to ILPs and LPs as well.


We consider PtO for MILPs , where the vector of cost parameters $\cost$ is to be predicted using
a vector of contextual information, $\feature$, correlated with $\cost$.
PtO problems comprise two steps: predicting unknown parameters given contextual information and solving the MILP with these predicted parameters.
In PtO problems, an ML model, $\ML$ (with trainable parameters $\theta$), is trained using past observation pairs $\{ (\feature_i, \cost_i) \}_{i=1}^N$ to map $\feature \rightarrow \cost$. Due to their successful performance in many predictive tasks, neural networks are commonly used as the predictive model in PtO settings \citep{aaai/WilderDT19}. We denote the predicted cost produced by the neural network, as $\predcost$, i.e., $\predcost = \ML(\feature)$.

A straightforward approach to the PtO problem is to train $\ML$ to minimize the prediction error between $\cost$ and $\predcost$.
However, previous works \citep{elmachtoub2022smart,mandi2020smart, aaai/WilderDT19} show that such a \emph{prediction-focused approach} can produce suboptimal decision performance.
By contrast, in DFL, the ML model is directly trained to optimize 
the decision loss, which reflects the quality of the resulting decisions.
When only the parameters in the objective function are predicted, \emph{regret}, which measures the suboptimality of a decision resulting from a prediction, is the decision loss of interest.
In DFL, one can consider other decision losses  (Appendix \ref{appendix:pto_des}), such as squared decision errors (SqDE) between $\solution$ and $\decisionvar$.
\emph{Regret} and SqDE can be written in the following form: 
\begin{equation}
\regret(\decisionvar, \cost) \coloneqq \cost^\top \decisionvar - \cost^\top \solution, \quad 
\mathit{SqDE}(\decisionvar, \cost) \coloneqq \|\solution(\cost) - \decisionvar\|^2 \nonumber
\label{eq:regret}
\end{equation}
The DFL approach trains $\ML$ by minimizing $\frac{1}{N}\sum_{i=1}^N \regret (\decisionvar^*(\ML(\feature_i)), \cost_i)$, the empirical risk minimization counterpart of $\mathbb{E}[\regret (\decisionvar^*(\ML(\feature)), \cost )]$.
This minimization of regret in gradient descent-based learning requires backpropagation through the optimization problem, which involves computing the Jacobian of $\solution(\predcost)$ with respect to $\predcost = \ML(\feature)$.
While $\frac{d \solution(\predcost)}{d \predcost}$ can be computed for convex optimization problems through implicit differentiation \citep{agrawal2019differentiable, amos2017optnet}, it raises difficulties when the optimization problem is combinatorial.
This is because when the parameters of a combinatorial optimization problem change, the solution either remains unchanged or shifts abruptly, meaning the gradients are zero almost everywhere, and undefined at transition points.

There are two primary approaches to implementing DFL for MILPs:
(a) smoothing the MILP to a differentiable convex optimization problem, and
(b) using a surrogate loss that is differentiable.
We briefly explain these approaches below. For more details on DFL techniques, we refer to the survey by \citet{mandi2023decision}.

\subsection{Differentiable Optimization by Smoothing of Combinatorial Optimization}\label{sect:smoothing}
To address the zero-gradient issue, methodologies in this category modify the optimization problem into a differentiable one by `smoothing' and then analytically differentiating the smoothed optimization problem (see Appendix \ref{Appendix:dflindetail} for a  detailed explanation).
For LPs, \citet{aaai/WilderDT19} propose transforming the LPs into `smoothed' QPs by augmenting the objective function with the square of the Euclidean norm of the decision variables.
Formally, they solve the following quadratically regularized QP:
\begin{equation}
     \label{eq:qptl}
     \min_{\decisionvar}  \predcost^\top \decisionvar + \mu \lVert \decisionvar \rVert^2_2  \;\; \text{s.t.} \;\;A \decisionvar = \mathbf{b}  \; ;
    \decisionvar \geq \bm{0}
\end{equation}
\noindent
\paragraph{Exact differentiable optimization of the QP.} The QP problem can be solved using differentiable solvers, such as \textsl{OptNet} \citep{amos2017optnet} or \textsl{Cvxpylayers} \citep{agrawal2019differentiable}.
The QP smoothing approach has been applied in various DFL works \citep{Mipaal,ferber2023predicting,mckenzie2024learning}. 
\citet{MandiNEURIPS2020} consider another form of smoothing by adding a logarithmic barrier term into the LP.
When the underlying optimization problem is an MILP or ILP, smoothing of the LP resulting from the continuous relaxation of the (MI)LP is carried out. Note that smoothing is applied \emph{only during training to enable backpropagation through the smoothed QP}; during testing and evaluation, the true MILP is solved.

When a differentiable solver like Cvxpylayers solves the QP, it computes the exact solution and the exact derivative by solving the QP using interior-point methods, and then differentiating the optimality conditions (i.e., the KKT optimality conditions). This requires computationally intensive matrix factorization such as LU decomposition.
\paragraph{Approximate differentiable optimization by DYS-Net.}
A recent method called DYS-Net \citep{mckenzie2024learning}
avoids the computational cost by not computing the exact solution and the exact derivatives.
To compute an approximation of the solution of Eq.~\ref{eq:qptl}, DYS-Net uses projected gradient descent \citep{Duchi}. However, projecting into the feasible space of an LP is itself a complex operation. 
This complex projection is avoided through a fixed-point iteration algorithm based on a three-operator splitting method \citep{davis2017three}. This version does not perform exact projections, but approximate projections. 
Note, for standard form LPs, the feasible space can be expressed as:
\[
\feas \equiv \feas_1 \cap \feas_2 \text{ where }
\feas_1 \doteq \{ A \decisionvar = b \} \text{ and } \feas_2 \doteq \{ \decisionvar \geq 0\}.
\]
Although projecting directly into $\feas $ is a complex task, projecting only into $\feas_1$ or $\feas_2$ are much simpler tasks, as shown below:
\begin{equation*}
P_{\feas_1} (\decisionvar) \doteq 
\decisionvar - A^{\dagger}(A \decisionvar - b) \text{ and } P_{\feas_2} (\decisionvar) \doteq \max \{ 0, \decisionvar \}
\end{equation*}
where $A^{\dagger}$ is the pseudo inverse of $A$ and $\max$ operates element-wise.
\citet{cristian2023end} propose continuously iterating between $P_{\feas_1}$ and $P_{\feas_2}$. 
In contrast, \citet{mckenzie2024learning} propose the following fixed-point iteration:
\begin{align}
    \decisionvar_{\iota+1} = 
    \decisionvar_\iota - P_{\feas_2} (\decisionvar_\iota) + P_{\feas_1} \big(  (2-\alpha \mu) P_{\feas_2}(\decisionvar_\iota) - \decisionvar_\iota - \alpha \cost \big)
    \label{eq:DYS_iteration}
\end{align}
\noindent which converges to $\solution(\cost)$ as $\iota \rightarrow \infty$. 

The first efficiency gain of DYS-Net comes from its use of this fixed-point iteration in the forward pass, to approximate the solution of Eq.~\ref{eq:qptl}.
The second efficiency gain is achieved in the gradient computation in the backward pass. 
Instead of obtaining the exact derivative by computing the inverse of the Jacobian of the fixed-point iteration,
%
they use Jacobian-free backpropagation \citep{fung2022jfb}, by replacing the Jacobian with an identity matrix. They show that this provides useful approximate gradients for backpropagation.  
This approximation turns Eq.~\ref{eq:DYS_iteration} into a series of matrix operations which can be implemented using standard neural network layers.
We denote the solution obtained by this method as $\dys (\cost)$.
%
After obtaining the solution using DYS-Net, \citet{mckenzie2024learning} minimizes $\mathit{SqDE}$ and backpropagates it through DYS-Net.
In contrast, we will propose to minimize surrogate losses after getting the solution using DYS-Net.
\subsection{Surrogate Losses for DFL}
Surrogate loss functions are used for training in DFL because 
they are crafted to have non-zero (sub)gradients while also reflecting the decision loss -- as regret decreases, surrogate loss functions decrease as well.
We discuss two surrogate losses, which we will use later.
\subsubsection{Smart Predict then Optimize Loss (SPO+)}\label{sec:sub_spo+}
The SPO+ loss \citep{elmachtoub2022smart}, a convex upper bound of $\regret (\solution (\predcost), \cost)$,
is one of the first and most widely used surrogate losses for linear objective optimization problems. 
Instead of minimizing $\regret$, they propose to minimize $\spoplus (\solution (\predcost), \cost)$, a convex upper bound of the regret. $\spoplus (\solution (\predcost), \cost)$ can be expressed in the following form:
\begin{align}
     \spoplus (\solution (\predcost), \cost) =(2\predcost - \cost)^\top \solution - (2\predcost - \cost)^\top \solution (2\predcost - \cost)
    \label{eq:spoplus}
\end{align}
\subsubsection{Contrastive Loss}
\label{sec:sub_con}
\citet{mulamba2020discrete} propose the following surrogate loss, $\m{L}_{SCE}^{\hat{y}} (\solution (\predcost), \cost)$ based on self-contrastive estimation (SCE)~\citep{gutmannnoise}.
%
\begin{align}
    \label{eq:SCE}
    \m{L}_{SCE}^{\hat{y}} (\solution (\predcost), \cost)    =  
    \predcost^\top \solution - \predcost^\top \solution (\predcost)
\end{align}
Note that this loss is similar to $\spoplus$, except that $2\predcost - \cost$ is replaced with $\predcost$.
One shortcoming of this loss is that for linear objectives, its minimum, which is zero, can be 
achieved either when $\solution(\predcost) = \solution$ or when $\predcost = 0$.
To prevent minimizing the loss by always predicting $\predcost = 0$, they further propose the following variant for linear objectives:
\begin{align}
    \label{eq:SCE_linear_objectives}
   \m{L}_{SCE}^{(\hat{y} - y)} (\solution (\predcost), \cost) =  (\predcost - \cost)^\top \solution -  (\predcost - \cost)^\top \solution (\predcost) \nonumber \\= 
    \predcost^\top \solution - \predcost^\top \solution (\predcost) + 
    \cost^\top \solution (\predcost)  - \cost^\top \solution (\cost)
\end{align}
Because in this work we focus on linear problems, we will primarily refer to $\m{L}_{SCE}^{(\hat{y} - y)}$ and will henceforth simply denote it by $\m{L}_{SCE}$, except when explicitly distinguishing between $\m{L}_{SCE}^{\hat{y}}$ and $\m{L}_{SCE}^{(\hat{y} - y)}$.

Nevertheless, computation of $\spoplus$ or $\SCE$ 
entails solving the MILP with the predicted $\predcost$ for every training instance in each epoch, which creates a significant computational burden. \citet{mandi2020smart} propose to minimize the LP relaxation for computing $\spoplus$.
\citet{mulamba2020discrete} address the computational issue by using solution caching instead of repeatedly solving the MILP.
The CaVE technique, proposed by \citet{cave24} for ILPs, minimizes a different surrogate loss, which measures the angle between the predicted cost vector and the ‘normal cone’ of the true optimal solution.
\section{The Gradients of Surrogate Loss Functions}\label{section:surrogateloss}
In this work, we investigate the use of surrogate losses in decision-focused learning, even when a differentiable optimization solver is used. After smoothing the optimization problem, the regret can be minimized directly using the solution of this layer. However, we explore whether surrogate losses still offer advantages in this setting.
\subsection{Surrogate Losses without Differentiable Optimization}\label{sect:surrogate_without}
Note that if we minimize $\SCE$ or $\spoplus$ using a non-differentiable solver, the solution, $\frac{\partial \solution (\cost) }{\partial \cost}$ cannot be computed.
In this case, the SPO+ and SCE losses are minimized using gradients, $\nabla_{\spoplus }$ and $\nabla_{\SCE }$  respectively, which can be expressed as follows:
\begin{align}
    \nabla_{\spoplus }  &= 2 ( \solution  - \solution(2\predcost - \cost) ) \\
     \nabla_{\SCE } &= ( \solution  - \solution(\predcost ) )
\end{align}
We highlight that in such cases 
$\solution (\predcost)$ will be treated as a constant for gradient computation. 
Because of this reason the gradient of $\SCE^{(\hat{y} - y)}$ (\ref{eq:SCE_linear_objectives}) would be same as the gradient of $\m{L}_{SCE}^{\hat{y}}$ (\ref{eq:SCE}). 
Hence, in the absence of a differentiable optimization layer, minimizing either of these losses by gradient descent results in the same outcome.
Next we will compare between minimizing $\spoplus$ and $\SCE$ using a non-differentiable solver.
First, in Theorem \ref{label:pretheorem}, we show that if $\nabla_{\SCE}$ or $\nabla_{\spoplus}$ becomes zero, then the solution to the true cost is also an optimal solution under the predicted cost, assuming $\solution(\cost)$ is unique.
\begin{theorem}
    \label{label:pretheorem}
        Suppose $\m{Y}_{ \mathit{SCE} } (\cost) = \{ \predcost:  \nabla_{ \SCE }  (\solution (\predcost), \cost) = 0 \}$ and
    $\m{Y}_{\mathit{SPO}^{+}} (\cost) = \{ \predcost:  \nabla_{ \spoplus }  (\solution (\predcost), \cost) = 0 \}$.
    Then, $\solution(\cost)$ is an optimal solution to (\ref{eq:standard_LP}) for any $\predcost \in \m{Y}_{\mathit{SCE}}(\cost)$ or $\predcost \in \m{Y}_{\mathit{SPO}^{+}}(\cost)$. (Proof is provided in Appendix \ref{proof_ppretheorem})
\end{theorem}
Next, we present Theorem \ref{theorem:1}, which shows that minimizing $\spoplus$ results in predicted cost parameters $\predcost$ that are more robust to perturbations, compared to minimizing $\SCE$ under the uniqueness assumption of the solution.
\begin{theorem}
\label{theorem:1} 
    For any $\predcost \in \m{Y}_{\mathit{SCE}}(\cost)$ or $\predcost \in \m{Y}_{\mathit{SPO}^{+}}(\cost)$, we define the perturbation threshold as the norm of the smallest perturbation $\bm{\Delta}$ such that 
    $\solution(\cost)$ is no longer an optimal solution to $\predcost+\bm{\Delta}$. Formally,
\begin{align}
    \Gamma(\predcost) \coloneqq 
    & \min_{\decisionvar' \in \feas\; ,\bm{\Delta}} \quad  \|\bm{\Delta}\|_2 \\
    & \text{s.t.} 
    \quad 
     (\predcost + \bm{\Delta})^\top \solution > (\predcost + \bm{\Delta})^\top \decisionvar'
    \nonumber
\end{align}
    Now let,
    \begin{equation}
\Gamma_{\mathit{SCE}} \coloneqq \min_{\predcost \in \m{Y}_{\mathit{SCE}}(\cost)} \Gamma(\predcost), \quad 
\Gamma_{\mathit{SPO}^{+}} \coloneqq \min_{\predcost \in \m{Y}_{\mathit{SPO}^{+}}(\cost)} \Gamma(\predcost)
\end{equation}
    Then $\Gamma_{ \mathit{SPO}^{+} } \geq \Gamma_{\mathit{SCE}}$.
\end{theorem}
While the proof is provided in the Appendix \ref{proof_theorem1}, we provide intuition for Theorem \ref{theorem:1} with a simple example. Suppose that we have to select the most valuable of two items, with true values 5 and 10.
Now, $\nabla_{ \spoplus } = \solution(2\predcost - \cost)$ becomes zero when the predicted value of the second item exceeds that of the first by at least half the difference in their true values, i.e., $(10 - 5)/2 = 2.5$.
In contrast, $\nabla_{\m{L}_{SCE}}$ becomes zero as soon as the predicted value of the second item slightly exceeds that of the first, e.g., if the corresponding predictions are 8 and 8.01. 
With such predictions, $\nabla_{\m{L}_{SCE}}$ is already zero, and there is no gradient signal to make the difference any larger. 
In this way, minimizing $\SCE$ with a non-differentiable solver can leave $\predcost$ stuck near such \emph{boundaries}, where slight perturbations yield a solution different from the true optimal.

\subsection{Surrogate Losses with Differentiable Optimization}
\label{sect:diffentiablesolversurrogate}
To use a differentiable solver, if the problem is an MILP,  it is first relaxed to an LP.
Moreover, the LP is non-smooth, as small changes in the cost parameter either leave the solution unchanged or cause abrupt shifts.
QP smoothing converts the non-smooth LP into a smooth, differentiable QP, allowing regret to be computed and differentiated by solving the smoothed QP using a differentiable optimization layer.
Existing DFL approaches under this category minimize the empirical regret of the smoothed problem, assuming that this reduces the expected regret on unseen instances.
However, through a close inspection of how the incorporation of smoothing changes the gradient landscape, we reveal a shortcoming in this approach.

The introduction of smoothing ensures that the solution transitions smoothly, rather than abruptly, near the original LP's transition points. However, the solution of the smoothed QP remains unchanged, or changes very slowly, in regions that are not near the transition points, 
provided that the smoothing strength is kept low as illustrated in Figure~\ref{fig:QP} in Appendix~\ref{Appendix:dflindetail}.
So, in this region, $\frac{d \solution (\predcost)}{d \predcost}$ 
is nearly zero.
When regret is minimized, its derivative with respect to $\predcost$ takes the following form:
\begin{align}
\label{eq:grad_regret}
     \frac{\partial \solution (\cost) }{\partial \cost}\Bigr|_{\substack{\cost=\predcost}}  \cost 
\end{align}
\noindent where $\frac{\partial \solution (\cost) }{\partial \cost}\Bigr|_{\substack{\cost=\predcost}}$ is computed by considering the smoothed optimization problem.
Similarly, if $\mathit{SqDE}$ were considered as the training loss, the derivative would be:
\begin{align}
\label{eq:squared_grad_regret}
     \frac{\partial \solution (\cost) }{\partial \cost}\Bigr|_{\substack{\cost=\predcost}}  (  \solution(\predcost) - \solution )  
\end{align}
As we illustrated above, smoothing addresses the non-differentiability at the transition points, but the derivative $\frac{d \solution (\predcost)}{d \predcost}$ still remains zero far from these points.
Hence, in both Eq.~\ref{eq:grad_regret} and Eq.~\ref{eq:squared_grad_regret}, the derivative remains zero across large regions of the parameter space, due to $\frac{\partial \solution (\cost) }{\partial \cost}\Bigr|_{\substack{\cost=\predcost}}$ becoming zero.
Consequently, training by gradient descent would fail to change  $\predcost$
despite $\predcost$ resulting in non-zero regret.

To prevent the derivative from \emph{vanishing far from the transition points, in this paper, we argue in favour of minimizing a surrogate loss even when $\mathit{SqDE}$ or $\regret$ can be directly minimized using a differentiable solver}.
For instance, when $\spoplus$ is minimized, the derivative of it after smoothing with respect to $\predcost$ would be:
\begin{align}
\label{eq:SPO_fullgrad}
    2(\solution - \solution(2\predcost - \cost) ) + 2 \frac{\partial \solution (\cost) }{\partial \cost}\Bigr|_{\substack{\cost=2\predcost - \cost}} ( \cost - 2\predcost) 
\end{align}
Similarly if $\SCE$ is minimized after smoothing, the resulting derivative would be:
\begin{align}
\label{eq:SCE_fullgrad}
    (\solution - \solution(\predcost)) + \frac{\partial \solution (\cost) }{\partial \cost}\Bigr|_{\substack{\cost=\predcost}} ( \cost - \predcost) 
\end{align}
\noindent
What sets Eq.\ref{eq:SPO_fullgrad} and Eq.\ref{eq:SCE_fullgrad} apart from Eq.~\ref{eq:grad_regret} is the term $(\solution - \solution(\predcost))$.
The term $(\solution - \solution(\predcost))$ prevents $\frac{d \m{L}}{d \predcost}$ going to zero even when $\frac{d \solution (\predcost)}{d \predcost} \approx 0$. 
%
\subsection{Numerical Illustrations}\label{sect:illustration}
To illustrate that zero-gradient issue persists even after QP smoothing, we will demonstrate how the gradient landscape changes after QP smoothing with a simple illustration.
For this, we consider the following one-dimensional optimization problem:
\begin{align}
\label{eq:relu_illustrate}
    \min_{w} y w & \; \; \text{s.t.}\;  0 \leq w \leq 1 \;  
\end{align}
where $y \in \mathbb{R}$ is the parameter to be predicted. 
Note that the solution of this problem is: $w^\star (y) = 1$ if $y < 0$ and $w^\star (y) = 0$ if $y > 0$; and $y=0$ for any value in the interval [0,1] is an optimal solution.
Let us assume that the true value of $y$ is $4$ and hence $w^\star (y) = 0$.
The \textcolor{red}{red} line in Figure~\ref{fig:smoothing}
shows how regret changes with predicted $\hat{y}$:
$4$ when $\hat{y}\leq0$ and $0$ when $\hat{y}>0$.  
The regret changes abruptly at $\hat{y}=0$.

\begin{figure}
\centering
    \includegraphics[scale=0.25]{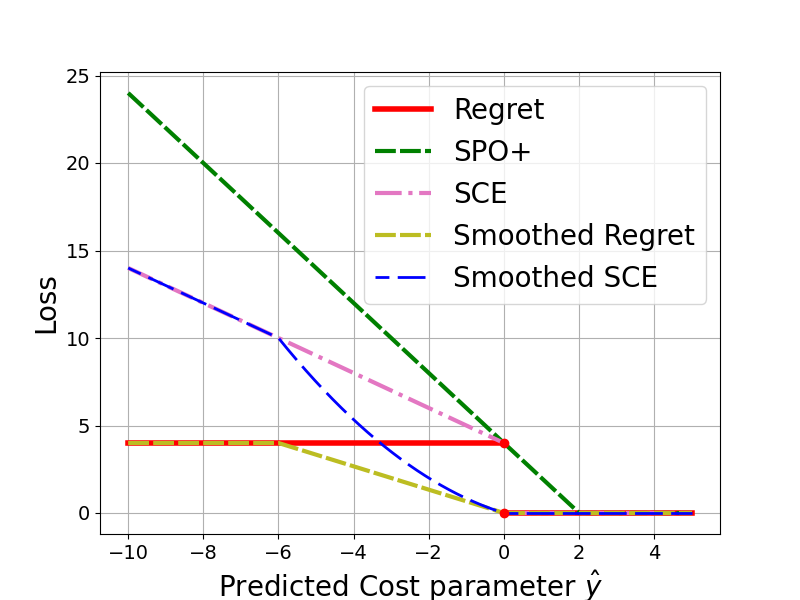}
    \vspace{0.5em}
  \caption{The numerical illustration shows that while smoothing causes solution plateaus with zero gradient. In contrast, $\SCE$ (with or without smoothing) ensures a non-zero gradient when regret is non-zero.}
  \vspace{1.5em}
  \label{fig:smoothing}
\end{figure}

The solution after augmenting the objective with  the quadratic smoothing term $\frac{\mu}{2} w^2$ with $\mu > 0$ is:
\begin{equation*}
w^\star (y) = 
\begin{cases}
0; & \text{when } y > 0\\
-\frac{y}{\mu}; & \text{when } -\mu < y \leq 0 \\
1; & \text{when } y \leq -\mu
\end{cases}
\end{equation*}
The regret with the smoothed QP is shown by the \textcolor{olive}{yellow} line in Figure~\ref{fig:smoothing} for $\mu = 6$.
The smoothing makes the derivative non-zero in the interval $-\mu \leq y \leq 0$, but zero when $y < -\mu$.
Note if $\hat{y} < - \mu$, \textit{the derivative of regret is $0$, even if regret is non-zero.}
Hence, the predictions cannot be changed by gradient descent despite regret being zero.
The smoothing strength can be increased by setting $\mu$ to a high value.
However, if $\mu \gg|y|$, $w^\star (y) \approx 0$ almost everywhere.

We plot $\m{L}_{SCE}$ with and without smoothing with \textcolor{blue}{blue} and 
\textcolor{violet}{violet} colors, respectively. In both cases, $\m{L}_{SCE}$ is strictly decreasing for $\hat{y} <0$ ensuring a non-zero derivative for $\hat{y} < 0$ and guiding $\hat{y}$ towards the positive half-space.
The fact that minimizing regret leads to zero gradients, limiting gradient-based learning, is further evidenced on a 2-dimensional LP in Appendix \ref{Appendix_simu}.
\begin{figure}[t]
    \centering
\includegraphics[width=0.9\linewidth]{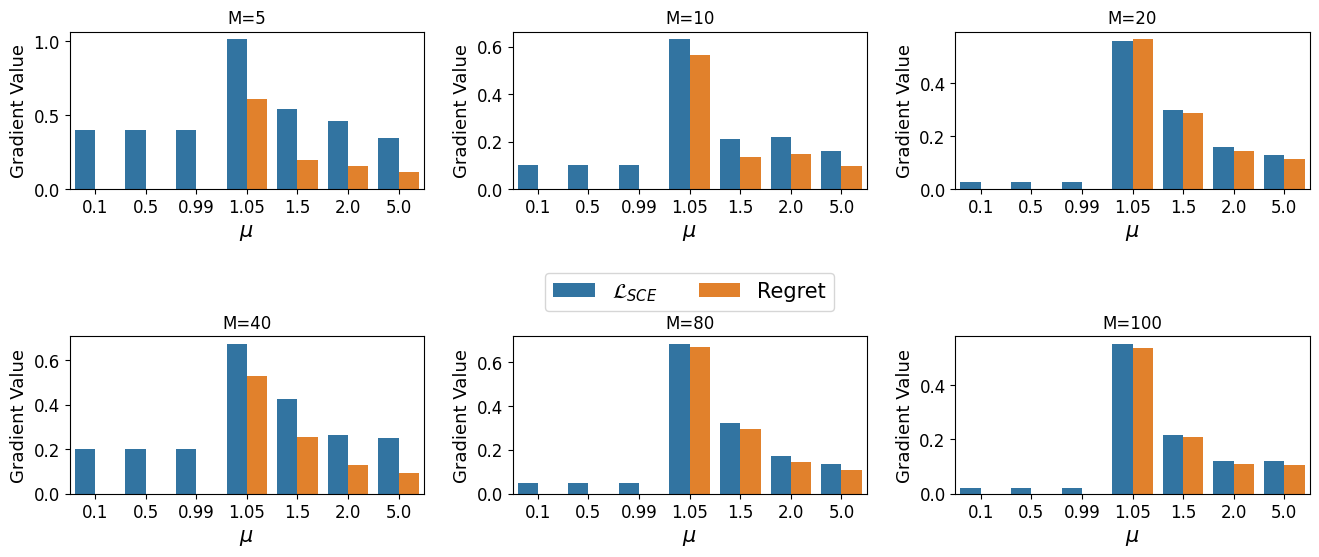}
    \caption{Results of computational simulation.}
    \vspace{1em}
    \label{fig:SimulationGradient_maintext}
    \vspace{1em}
\end{figure}

\subsection{Impact of Differentiable Optimization on $\SCE$} 
In Section \ref{sect:surrogate_without}, we argue that minimizing $\spoplus$ provides better robustness than minimizing $\SCE$ \emph{with a non-differentiable solver}. Recall that this is due to the fact that there is not much gradient signal to move away from the \textit{boundary} when $\SCE$ is minimized without differentiable optimization.
However, minimizing with a differentiable solver can prevent this.
In this case, if $\predcost$ lies in such \emph{boundary}, the term, $\frac{\partial \solution (\cost) }{\partial \cost}\Bigr|_{\substack{\cost=\predcost}} ( \cost - \predcost) $ will be non-zero and this part of the gradient will push $\predcost$ away from the \emph{boundary}. 
On the other hand, if $\predcost$ is far from the \emph{boundary} and still
$ \solution(\predcost) \neq \solution$, the term, $\frac{\partial \solution (\cost) }{\partial \cost}\Bigr|_{\substack{\cost=\predcost}} \approx 0$, but  the first term of the gradient will drive $\predcost$ to make  $ \solution(\predcost) = \solution$.
In principle, both $\SCE$ or $\spoplus$ can be minimized using a differentiable solver. 
However, we want to highlight the following properties of $\SCE$ in Proposition \ref{prop:1}.
\begin{proposition}
\label{prop:1}
For all, $\cost, \predcost \in \mathbb{R}^K$, the following holds:
   \begin{enumerate}
       \item $\SCE (\solution (\predcost), \cost) \geq 0$
       \item When the set of optimal solutions is a singleton\\ $\SCE (\solution (\predcost), \cost) = 0 \Leftrightarrow \regret (\solution (\predcost), \cost) = 0$.
   \end{enumerate}
   (A proof is provided in Appendix \ref{proof_prop1}. We also show generalization bounds for $\SCE$ loss in Appendix \ref{generalization_bounds}. )
\end{proposition}
\noindent  Proposition 1 shows that $\SCE$ is zero whenever $\regret$ is zero. However, $\spoplus$ is an upper-bound of $\regret$ and can be non-zero, even when $\regret$ is zero. We can see this in Figure~\ref{fig:smoothing}.
Another consideration is that $\spoplus$ is a convex loss even without smoothing, whereas $\SCE$ is not, as it has discontinuity. However, it becomes a convex loss after smoothing.
We will empirically investigate which one, after smoothing, is more useful for gradient-based DFL.
 
 

Furthermore, in Section \ref{sect:surrogate_without}, we showed that $\m{L}_{SCE}^{\hat{y}}$ and $\SCE^{(\hat{y} - y)}$ produces the same gradient with a non-differentiable solver. 
In contrast, with a differentiable solver, the gradient of $\m{L}_{SCE}^{\hat{y}}$ differs from that of $\SCE^{(\hat{y} - y)}$, as shown below:
$$
(\solution - \solution(\predcost)) - \frac{\partial \solution (\cost) }{\partial \cost}\Bigr|_{\substack{\cost=\predcost}} (  \predcost)$$
Also, note that if $\SCE$ or $\spoplus$ is minimized with a non-differentiable solver, $\frac{\partial \solution (\cost) }{\partial \cost}$ cannot be computed, so Eq.\ref{eq:SPO_fullgrad} and Eq.\ref{eq:SCE_fullgrad} reduce to SPO+ and SCE subgradients, respectively.


\section{Simulation-Based Gradient Analysis on LPs}
Through simulations on synthetically generated data, in this section, we will further demonstrate
that the zero-gradient issue exists even for quadratically regularized LPs larger than one dimension. 
For the simulations, we consider Top-1 selection problems with different numbers of items $M$.
This can be represented as the following LP:
\begin{equation}
    \label{eq:TopkSimu}
     \max_{\decisionvar \in \{0,1\} } \cost^\top \decisionvar \;\;\text{s.t.}\; \decisionvar^\top \mathbf{1}\leq 1
\end{equation}
\noindent Here, $\cost =[y_1, \ldots, y_M]  \in \mathbb{R}^M$ is the vector denoting value of all the items and $\decisionvar = [w_1, \ldots, w_M]$ is the vector of decision variables.
To replicate the setup of a PtO problem, we solve the optimization problem with $\predcost$, compute $\regret$ and $\SCE$, and then analyze the corresponding gradients.
To generate the ground truth $\cost$, we randomly select $M$ integers without replacement from the set $\{1, \dots, M\}$. The predicted costs, $\predcost$, are generated by considering a different sample from the same set. As a result, $\cost$ and $\predcost$ contain the same numbers but in different permutations.
We solve the LP with $\predcost$, after adding the quadratic regularizer $\mu$, using \textsl{Cvxpylayers}.

We compute the gradients of $\regret$ and $\SCE$ for multiple values of $M$ and $\mu$.
For each configuration of $M$ and $\mu$, we run 20 simulations.
and show the average absolute values of the gradients of the two losses -- $\SCE$ and $\regret$ --  in Figure \ref{fig:SimulationGradient_maintext}. We also show the average Manhattan distance between solutions of the true LP and `smoothed' QP for same $\predcost$ in Table \ref{tab:smilate_table_maintext}.
As we hypothesized, the gradient turns zero whenever $\regret$ is minimized with $\mu <1$. This is not the case when $\SCE$ is minimized.
We refer readers to Appendix~\ref{Appendix_simu} for a detailed description of the simulation setup and analysis.
\begin{table}[]
    \centering
     \caption{We tabulate average Manhattan distance between the the LP and the smoothed QP solutions for different values of $M$ and $\mu$.}
\begin{tabular}{lrrrrrrrr}
\toprule
& \multicolumn{6}{c}{M}\\  
\cmidrule{2-7}
$\mu$ & 5 & 10 & 20 & 40 & 80 & 100   \\
\midrule
0.100& 0.000& 0.000& 0.000& 0.000& 0.000& 0.000\\
0.500& 0.000& 0.000& 0.000& 0.000& 0.000& 0.000\\
0.990& 0.000& 0.000& 0.000& 0.000& 0.000& 0.001\\
1.050& 0.089& 0.089& 0.089& 0.089& 0.089& 0.089\\
1.500& 0.465& 0.466& 0.465& 0.465& 0.465& 0.464\\
2.000& 0.622& 0.622& 0.622& 0.622& 0.622& 0.622\\
5.000& 1.165& 1.165& 1.165& 1.165& 1.165& 1.165\\
\bottomrule
\end{tabular}
   
    \label{tab:smilate_table_maintext}
\end{table}



\begin{figure*}[htb]
    \centering
    \begin{subfigure}{0.45\textwidth}
        \centering
        \includegraphics[width=\textwidth]{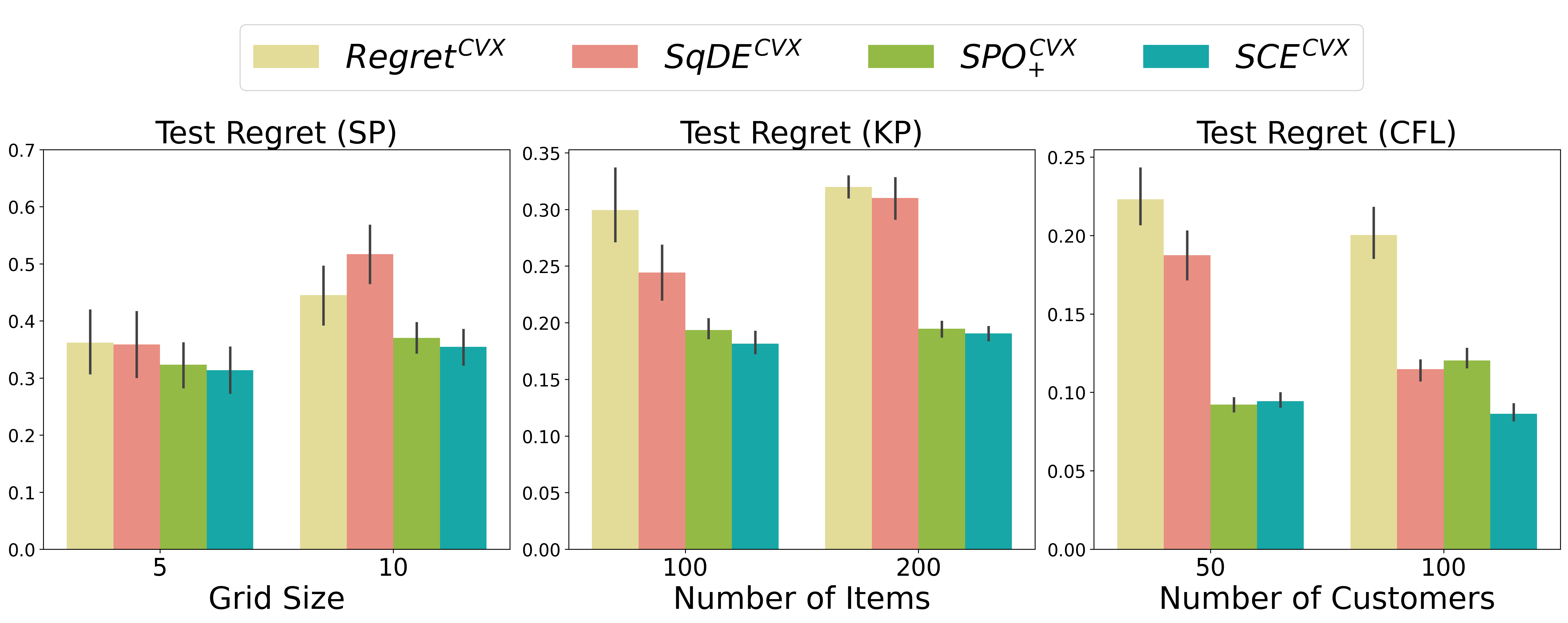}
        \caption{Cvxpylayers}
        \label{fig:cvx}
    \end{subfigure}
    \hfill
    \begin{subfigure}{0.45\textwidth}
        \centering
        \includegraphics[width=\textwidth]{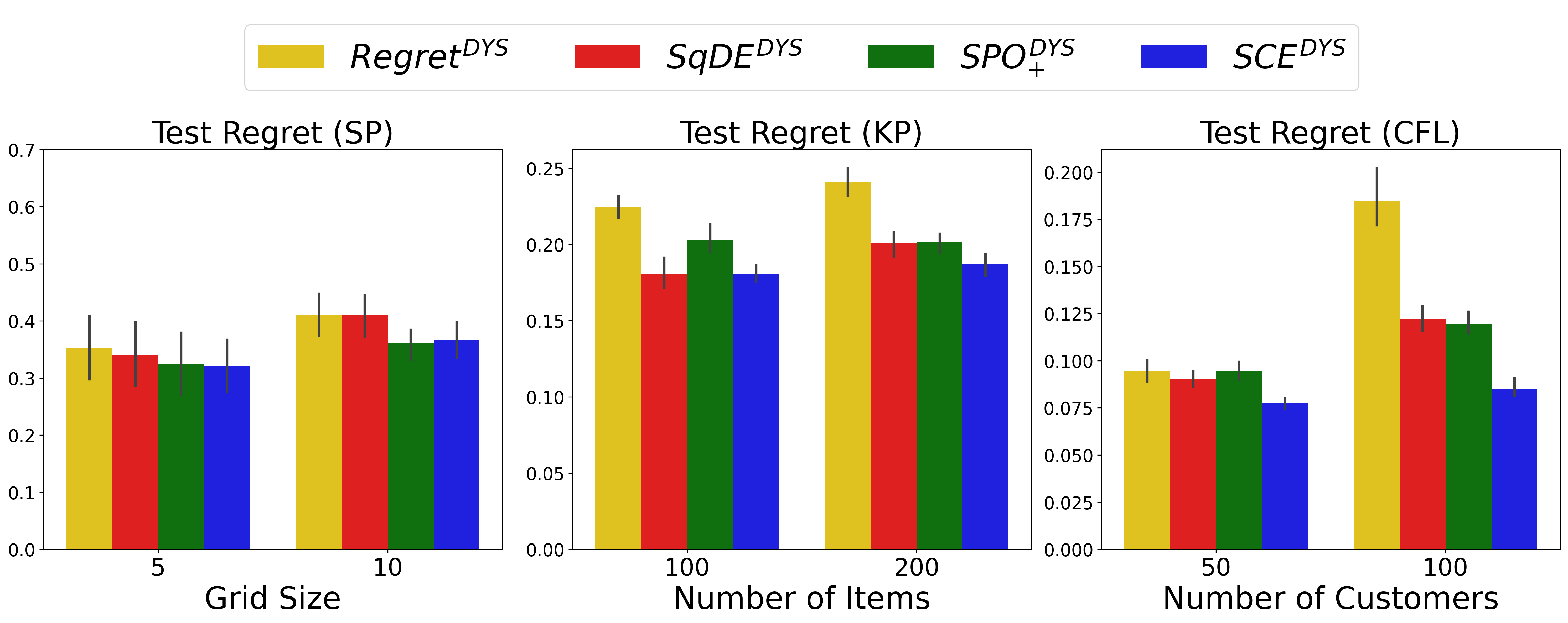}
        \caption{DYS-Net}
        \label{fig:dysnet}
    \end{subfigure}
    \vspace{1em} 
    \caption{Minimizing surrogate losses versus $\regret$ or $\mathit{SqDE}$ using differentiable optimization.}
    \vspace{1em}
    \label{fig:differtiableOpt}
\end{figure*}
\section{Experimental Evaluation}
We will consider the following three optimization problems in our experiments. We provide detailed descriptions and mathematical formulations in Appendix~\ref{appendix:ILP_formulation}.
See Appendix \ref{appendix:expsetup} for hyperparameters and tabulated results. The source code of our implementation is publicly available at \url{https://github.com/JayMan91/DYS-NET-SCE}.
\paragraph{Shortest path on a grid (SP).}
The goal of this optimization problem is to find the path with lowest cost on a $k \times k$ grid, starting from the southwest node and ending at the northeast node of the grid \citep{elmachtoub2022smart}.
A SP problem is characterized by the size of the grid, $k$.
The cost of each edge is unknown and should be predicted before solving the problem.  
This problem can be solved as an LP, due to total unimodularity~\citep{wolsey}. 
\paragraph{Multi-dimensional knapsack (KP).}
In this problem, a maximal value subset of items must be selected while respecting two-dimensional capacity constraints. The weights and capacities are known, and the item values should be predicted. Each KP instance is defined by the number of items. The problem is an ILP. 

\paragraph{Capacitated facility location (CFL).}
Given a set of feasible facility sites and a set of customers, the goal is to satisfy the customers' demands for a single product while minimizing the total cost. The total cost includes both the fixed costs of opening the facilities and the transportation costs. The fixed costs, customer demands, and facility capacities are assumed to be known. In the PtO version, the transportation costs are considered unknown. The problem is a MILP.

\subsection{Experimental Setup}
The training, validation and test instances for the SP, KP and CFL problems are generated using \textsl{PyEPO}~\citep{pyepo}.
In all problems considered, the true relationship between features $\feature$ and costs $\cost$ is non-linear, but we use linear models to predict $\cost$. This setup, common in PtO evaluations, highlights how DFL methods can still achieve low regret even when the predictive model is misspecified.
Appendix \ref{appendix:data} explains the true underlying relationship between the cost and the features. 
We use the polynomial degree parameter and the noise half-width parameter as 6 and 0.5, respectively, in all our experiments.
The predictive models are implemented using \textsl{PyTorch} \citep{paszke2019pytorch}, and \textsl{Gurobipy} \citep{gurobi} is used to implement the (MI)LPs.
For evaluation, we always compute the optimal solution both for true and predicted costs using the (MI)LP solvers.
%
For all the experiments, we report the average of \textit{normalized relative regret} on test instances, calculated as follows:
\begin{equation}
\label{eq:relative_regret}
\frac{1}{\mathit{N_{test}}} \sum_{i=1}^{ \mathit{N_{test}} } \frac{ \cost_i^\top (\solution (\predcost_i) - \solution_i) } {\cost_i^\top \solution_i  }.
\end{equation} 
\noindent where $\mathit{N_{test}}$ is the number of test instances. Furthermore, each experiment is repeated five times with different seeds, and the average is reported.
%
%
The experiments were executed on an \textsl{Intel i7-13800H} (20 cores) CPU with 32GB RAM.

\begin{figure*}[htb]
    \centering
    \begin{subfigure}{0.49\textwidth}
        \centering
        \includegraphics[width=\textwidth]{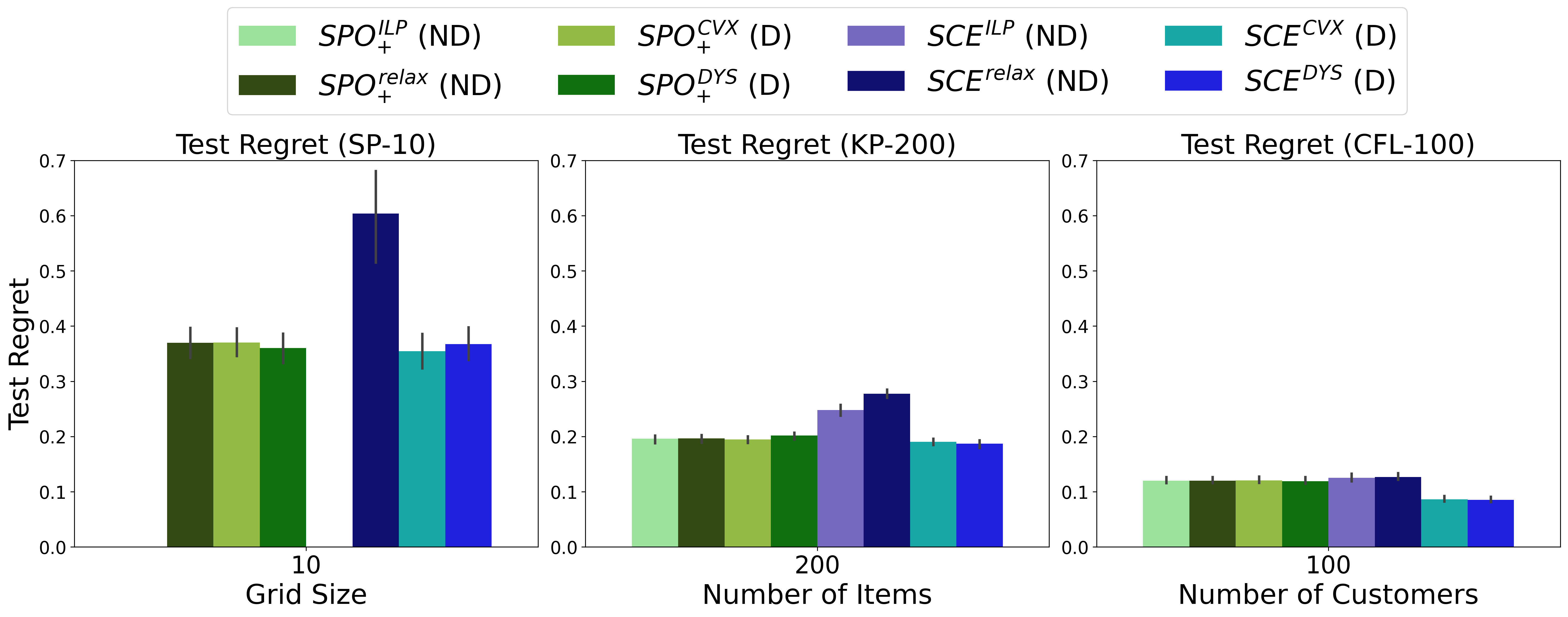}
        \caption{Test regret comparison when ILP and LP relaxation are solved by differentiable (D) and non-differentiable (ND) optimizer.}
        \label{fig:RQ2Regret}
    \end{subfigure}
    \begin{subfigure}{0.49\textwidth}
        \centering
        \includegraphics[width=\textwidth]{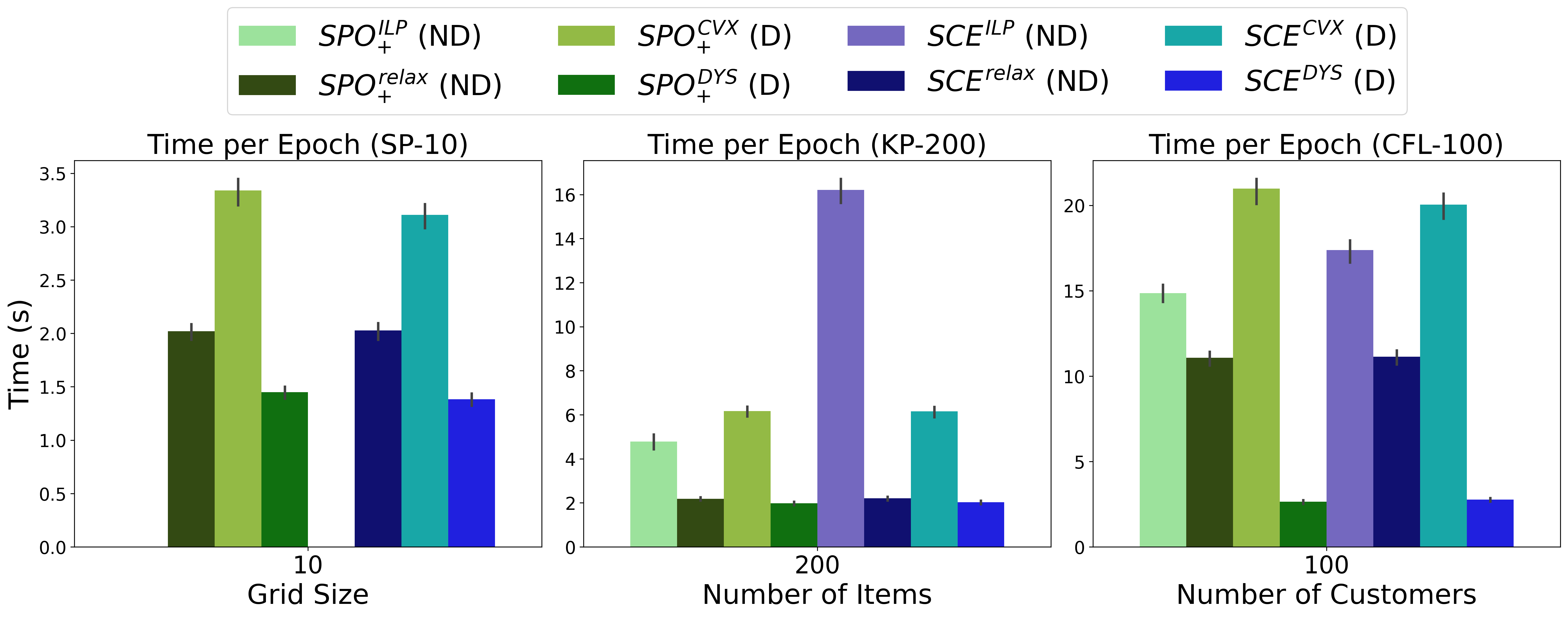}
        \caption{Runtime comparison when ILP and LP relaxation are solved by differentiable (D) and non-differentiable (ND) optimizer.}
        \label{fig:RQ2Runtime}
    \end{subfigure}
\vspace{1.5em} 
    \caption{Minimizing surrogate losses using differentiable and non-differentiable optimizer.}
    \vspace{2em}
    \label{fig:RQ2}
\end{figure*}
\subsection {Experimental Results}
\subsubsection{Minimizing Surrogate Losses vs. Decision Losses with Differentiable Solver}
In the first set of experiments, we investigate whether minimizing surrogate losses in the presence of differentiable optimization layers leads to lower regret compared to directly minimizing $\regret$ or $\mathit{SqDE}$.
We consider SP instances with gridsize 5 and 10; KP instances with 100 and 200 items, and CFL instances with 50 and 100 customers.
In Figure~\ref{fig:cvx}, $Regret^{CVX}$, $SqDE^{CVX}$, $SPO_{+}^{CVX}$, and $SCE^{CVX}$ display the average test regret after minimizing $\regret$, $\mathit{SqDE}$, $\spoplus$, and $\SCE$ using Cvxpylayers as a differentiable solver, respectively.
In Figure~\ref{fig:dysnet}, these losses are minimized using DYS-Net as the differentiable solver.
Figures \ref{fig:cvx} and \ref{fig:dysnet} show that minimizing $\SCE$ or $\spoplus$ consistently yields lower regret than minimizing $\regret$ or $\mathit{SqDE}$, both when using an exact differentiable solver (Cvxpylayers) and an approximate one (DYS-Net).
\subsubsection{Minimizing surrogate losses using different formulations.}
Previous experiments showed that, when using differentiable optimization, minimizing surrogate losses yields lower test regret than directly minimizing regret.
We now compare the quality and runtime of minimizing surrogate losses with different kinds of solvers.
We will minimize both $\spoplus$ and $\SCE$ using two non-differentiable solvers -- one solving the ILP formulation (for ILP problems) and one solving the LP relaxation; as well as two differentiable solvers: Cvxpylayers (an exact differentiable solver) and DYS-Net (an approximate, faster differentiable solver), which solve the `smoothed' QP.
In Figure \ref{fig:RQ2}, for SP, we do not consider the ILP formulation, as the LP formulation, itself, provides the exact solution. 

In Figure \ref{fig:RQ2Regret}, minimizing $\spoplus$ yields similar test regret regardless of the solver type. In contrast, minimizing $\SCE$ with a non-differentiable solver leads to significantly higher regret, consistent with Theorem \ref{theorem:1}, which states that $\spoplus$ is less sensitive to perturbations, while $\SCE$ can result in non-zero regret from slight changes.
As noted in Section \ref{sect:diffentiablesolversurrogate}, this issue with $\SCE$ is mitigated by using a differentiable solver. Figure \ref{fig:RQ2Regret} confirms this: minimizing $\SCE$ with a differentiable solver gives the lowest test regret, even lower compared to minimizing $\spoplus$, whether $\spoplus$ is minimized using a differentiable or non-differentiable solver.
Figure \ref{fig:RQ2Runtime} reports training times. DYS-Net is consistently faster.
Thus, minimizing $\SCE$ with DYS-Net achieves both low regret and low training time.
\begin{figure*}[h]
    \centering
    \begin{subfigure}{0.49\textwidth}
        \centering
        \includegraphics[width=\textwidth]{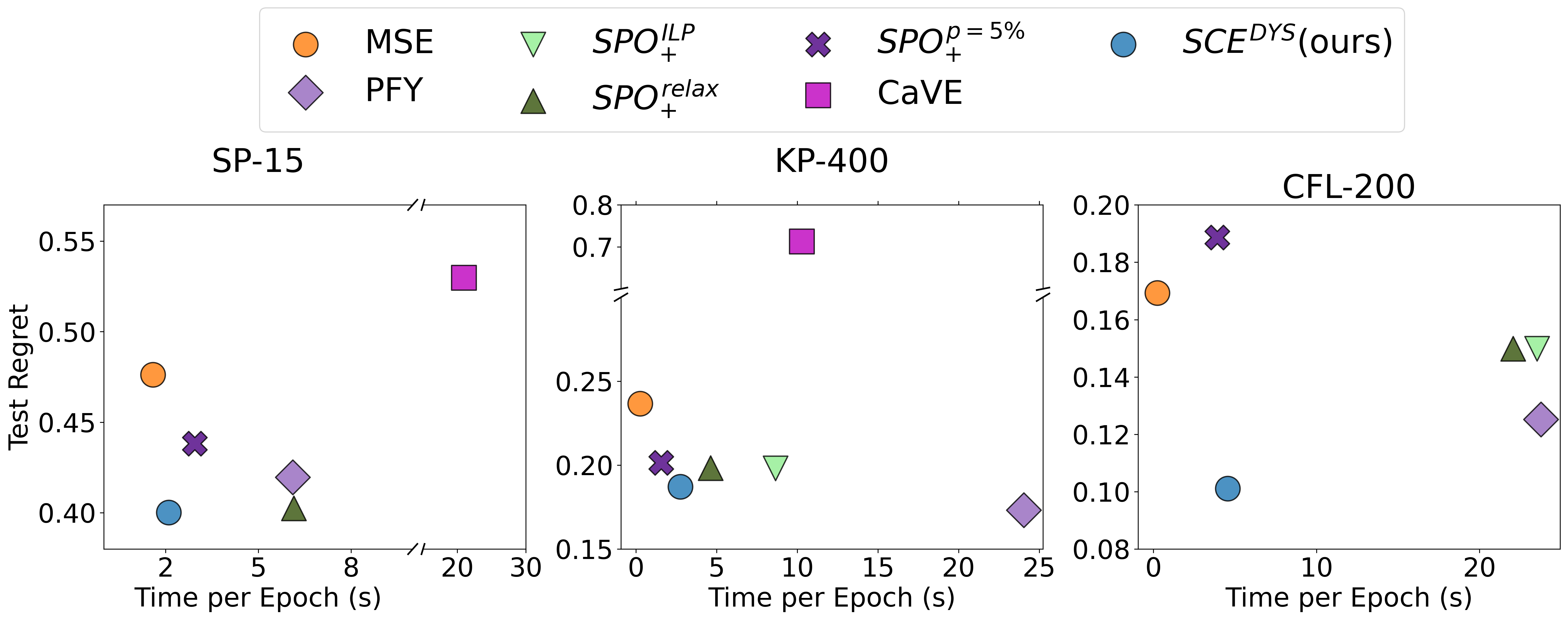}
        \caption{}
        \label{fig:sub1}
    \end{subfigure}
    \begin{subfigure}{0.49\textwidth}
        \centering
        \includegraphics[width=\textwidth]{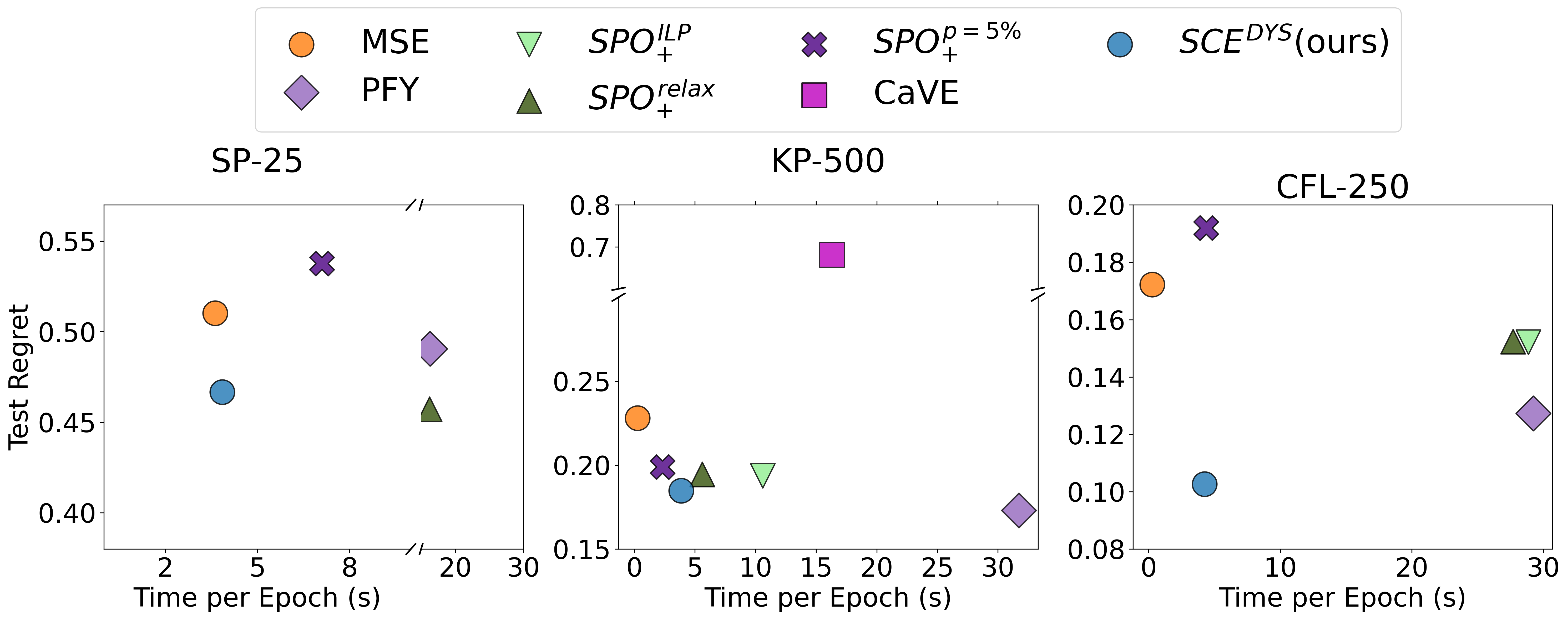}
        \caption{}
        \label{fig:sub2} 
    \end{subfigure}
    \vspace{1em}
    \caption{Comparison of Regret and Runtime against the state-of-the-art DFL techniques.}
    \vspace{1em}
    \label{fig:SOTA}
\end{figure*}

\subsubsection{Comparison against the state of the art.}
The previous experiments show that minimizing $\SCE$ with a differentiable solver achieves regret comparable to minimizing $\spoplus$ with a non-differentiable solver, one of the state-of-the-art DFL techniques. As DYS-Net is significantly faster than Cvxpylayers and the non-differentiable exact solvers, we propose minimizing $\SCE$ using DYS-Net ($SCE^{DYS}$) to speed up DFL training without compromising quality.
Hence in the final set of experiments, we compare the performance of $SCE^{DYS}$ against `Perturbed Fenchel-Young' (PFY) loss \citep{berthet2020learning}, the CaVE loss \citep{cave24},
%
$SPO_{+}^{relax}$ and $SPO_{+}^{ILP}$, which minimize the $\spoplus$ loss with and without LP relaxation, respectively, 
and the $SPO_{+}^{p=5\%}$ variant, which employs solution caching to solve (MI)LPs only 5\% of the time \citep{mulamba2020discrete}. 

We evaluate performance on larger problem instances: SP instances with gridsize 15 and 20; KP instances with 400 and 500 items, and CFL instances with 200 and 250 customers. In Figure~\ref{fig:SOTA}, we plot the training time per epoch (x-axis) against normalized test regret (y-axis) to assess both solution quality and training efficiency.
Minimizing MSE yields the lowest training time per epoch since no (MI)LPs are solved, but this often leads to higher test regret compared to DFL methods as it consider the downstream (MI)LPs while training

In the SP instances, $SPO_{+}^{relax}$ achieves the lowest regret. The regret of $SCE^{\text{DYS}}$ is second-lowest. Moreover, it requires only \emph{one-third} the training time of $SPO_{+}$.
$SPO_{+}^{p=5\%}$ which solves the LP only 5\% of times, is faster than $SPO_{+}$ but not as fast as $SCE^{\text{DYS}}$ and its regret is also higher than $SCE^{\text{DYS}}$.
For the KP instances, PFY achieves the lowest regret but it has the longest runtime. 
$SCE^{\text{DYS}}$ is ranked second in terms of test regret -- while being nearly five times faster than PFY. Although $SPO_{+}^{p=5\%}$ runs faster than $SCE^{\text{DYS}}$, its regret is slightly worse. $SCE^{\text{DYS}}$ achieves both the \emph{lowest runtime and the lowest regret} among the DFL techniques in the CFL instances. While $SPO_{+}^{p=5\%}$ has comparable runtime in these instances, its test regret is substantially higher than $SCE^{\text{DYS}}$.
CaVE operates on active constraints, which grow with problem size in SP and CFL. This makes CaVE memory-intensive, and we are unable to run it on all instances.

To summarize, $SCE^{\text{DYS}}$ consistently matches or outperforms state-of-the-art DFL techniques, such as $SPO_{+}$ and PFY, in terms of regret, and significantly reduces training time across all problem instances, often by a factor of three or more compared to these methods. These results demonstrate that $SCE^{\text{DYS}}$ achieves low regret with faster training times, suggesting that it is an efficient method for DFL.

\section{Conclusion}
In this paper, we study the technique of making an (MI)LP differentiable by adding a quadratic regularizer to its objective. This technique is used in DFL to compute the Jacobian of the (MI)LP solution with respect to the cost parameters, and train an ML model to minimize regret by backpropagating through the smoothed QP.
We first identify the zero-gradient issue exists even after smoothing for both small and large LPs through both numerical examples and simulations. To circumvent this issue, we propose minimizing a surrogate loss instead of the regret. We empirically show that minimizing the surrogate loss, $\SCE$, with a differentiable solver results in regret as low as or better than that of state-of-the-art surrogate-loss-based DFL techniques. Finally, by using DYS-Net, a fast approximate differentiable solver, to minimize $\SCE$, we demonstrate that DFL training time can be reduced by a factor of three without sacrificing decision regret.

A recent study~\citep{shallpass} shows that the zero-gradient issue also occurs in problems with non-linear objectives. In future work, we plan to explore surrogate losses for such settings. Another direction for future work is to study ILPs with weaker LP relaxations, such as the MTZ formulation of the TSP \citep{sherali2002tightening}, and examine how relying on the relaxation impacts decision quality when using DYS-Net. Future work could also explore alternative smoothing techniques or surrogate losses that could further improve the effectiveness of this approach.



\begin{ack}
This research received funding from the European Research Council (ERC) under the European Union’s Horizon 2020 research and innovation program (Grant No. 101002802, CHAT-Opt). Senne Berden
is a fellow of the Research Foundation-Flanders (FWO-Vlaanderen, 11PQ024N).
\end{ack}



\bibliography{mybibfile}
\newpage
\appendix
\onecolumn
\appendix
\section{Predict-then-Optimize Problem Description}
\label{appendix:pto_des}
We consider predicting parameters in the objective function of an LP. These kinds of problems can be framed as \textit{predict-then-optimize} (PtO) problems consisting of a prediction stage followed by an optimization stage, as illustrated in Figure \ref{fig:problem_desc}.
In the prediction stage, an ML model $\ML$ (with trainable parameters $\omega$) is used to predict unknown parameters using features, $\feature$, that are correlated to the parameter. During the optimization stage, the problem is solved with the predicted parameters. 
An offline dataset of past observations is available for training $\ML$.

It is important to distinguish datasets based on whether the true parameters, $\cost$, are observed and included in the dataset.  In some applications, the true parameters, $\cost$, may not be directly observable, and only the solutions, $\solution(\cost)$, are observed. While $\solution(\cost)$ can be computed if $\cost$ is known, the reverse is not true, since solving the inverse optimization problem is a separate research area.

Whether $\cost$ is observed or not is important because 
in order to compute $\regret$ (\eqref{eq:regret}), we need the true parameter $\cost$.
Most of the benchmarks in PtO problems assume that $\cost$ is observed in the past observation.
In this case the training data can be expressed as $\{ (\feature_i, \cost_i, \solution(\cost_i)) \}_{i=1}^N$ and the empirical regret, $\frac{1}{N}\sum_{i=1}^N \regret (\decisionvar^*(\ML(\feature_i)), \cost_i)$, can be computed.
In most PtO benchmark problems it is assumed that the true $\cost$ is observed in the training data \citep{mandi2023decision, pyepo}.
However, if the true cost $\cost$ is not observed in the training data, empirical regret cannot be computed. 
Instead, some other loss must be considered.  
For instance, \citet{mckenzie2024learning} consider
squared decision errors (SqDE) between $\solution (\cost)$ and $\solution (\predcost)$, i.e., $\mathit{SqDE} = || \solution (\cost) - \solution (\predcost) ||^2$.
\label{Appendix_problem}
\begin{figure}
    \centering
    \begin{tikzpicture}[
  font=\rmfamily\footnotesize,
  every matrix/.style={ampersand replacement=\&,column sep=2cm,row sep=.6cm},
  source/.style={draw,thick,rounded corners,inner sep=.3cm},
  process/.style={draw,thick,ellipse, minimum width=1pt,
    align=center},
  dots/.style={gray,scale=2},
  to/.style={->,>=stealth',shorten >=1pt,semithick,font=\rmfamily\scriptsize},
  every node/.style={align=center}, 
  arrow/.style = {thick,-stealth}]

    \node[source] (z) at (-3, 0) {$\feature$}; 
     \node[process](ML) at (0,0){$\ML$}; 
     \node[source] (param) at (3, 0) {$\predcost$};
     \node[process](Opt) at (6,0){Combinatorial \\Optimization}; 
     \node[source](decision) at (9,0){Optimal \\Decisions}; 


     \draw [arrow] (z) --(ML);
     \draw [arrow] (ML) --(param);
     \draw [arrow] (param) --(Opt);
     \draw [arrow] (Opt) --(decision);
\end{tikzpicture}
  
    \caption{Schematic diagram of a predict-then-optimize (PtO) problem. }
    \label{fig:problem_desc}
\end{figure}
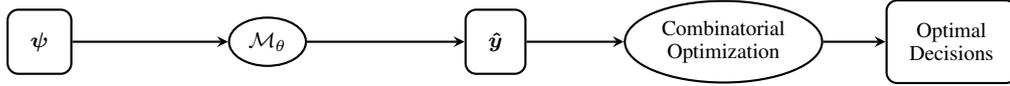

 \begin{algorithm}[H]
\caption{Gradient-descent with Smoothing}
\label{alg:smoothing}
\begin{algorithmic}[1] 
\STATE Initialize $\bm{\omega}$.
\FOR{each epoch}
\FOR{each instance $(\feature, \cost, \solution(\cost))$}
\STATE $\predcost = \ML (\feature)$ \label{alg:ln:pred} 
\STATE Obtain $\solution (\predcost) $ by solving a `smoothed' optimization\\
\STATE $\regret (\decisionvar, \cost)= \cost^\top \solution (\predcost)  - \cost^\top \solution(\cost)$ \\
\STATE $\bm{\omega} \leftarrow \bm{\omega} - \alpha \frac{d \regret (\decisionvar, \cost)  }{d  \predcost} \frac{d  \predcost }{d  \bm{\omega}}$

\ENDFOR
\ENDFOR
\end{algorithmic}
\end{algorithm}

\begin{algorithm}[H]
\caption{Gradient-descent with Surrogate Losses}
\label{alg:surrogate}
\begin{algorithmic}[1] 
\FOR{each epoch}
\FOR{each instance $(\feature, \cost, \solution(\cost))$}
\STATE $\predcost = \ML (\feature)$ \label{alg:ln:pred} 
\STATE Compute $\bm{\Tilde{\cost}}$\\
\STATE Obtain $\solution (\bm{\Tilde{\cost}}) $ by solving the original optimization\\
\STATE Compute the surrogate loss $\m{L}$ and $\nabla$
\STATE $\bm{\omega} \leftarrow \bm{\omega} - \alpha \nabla \frac{d  \predcost }{d  \bm{\omega}}$

\ENDFOR
\ENDFOR
\end{algorithmic}
\end{algorithm}

\begin{algorithm}[H]
\caption{Gradient-descent when Surrogate Losses are minimized using Smoothed Solver}
\label{alg:smoothedsurrogate}
\begin{algorithmic}[1] 
\FOR{each epoch}
\FOR{each instance $(\feature, \cost, \solution(\cost))$}
\STATE $\predcost = \ML (\feature)$ \label{alg:ln:pred} 
\STATE Compute $\bm{\Tilde{\cost}}$\\
\STATE Obtain $\solution (\bm{\Tilde{\cost}}) $ by solving a `smoothed' optimization\\
\STATE Compute the surrogate loss $\m{L}$ 
\STATE $\bm{\omega} \leftarrow \bm{\omega} - \alpha \frac{d \m{L}  }{d  \predcost}  \frac{d  \predcost }{d  \bm{\omega}}$

\ENDFOR
\ENDFOR
\end{algorithmic}
\end{algorithm}

\section{Different Approaches to  Decision-Focused Learning}
\label{Appendix:dflindetail}
In PtO problems, the empirical regret can be calculated if the cost, $\cost$, is observed in the training instances. However, just because it can be calculated does not mean it can be minimized using gradient descent.
Figure \ref{fig:dfl} illustrates the impact of integrating the optimization block into the training loop of neural networks. The key challenge is that to directly minimize $\regret$, it must be backpropagated through the optimization problem.
However, for a combinatorial problem $\solution(\predcost)$ does not change smoothly with $\predcost$, so the gradient, $\frac{d \solution(\predcost)}{d \predcost}$, is either zero or does not exist.
\begin{figure*}
    \centering
    \begin{subfigure}[b]{0.5\textwidth}
    \centering
    \includegraphics[scale=0.25]{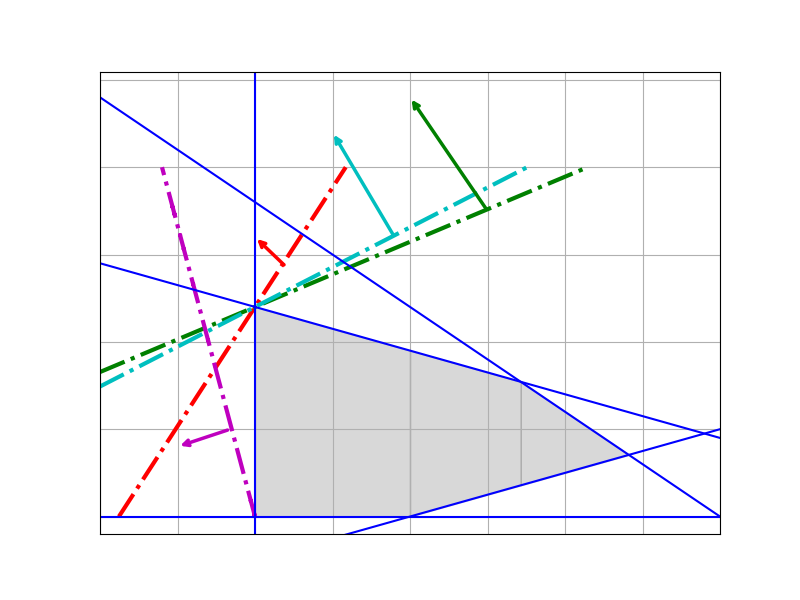}
    \caption{LP Solutions}
    \label{fig:LP}
  \end{subfigure}%
  \begin{subfigure}[b]{0.5\textwidth}
    \centering
    \includegraphics[scale=0.25]{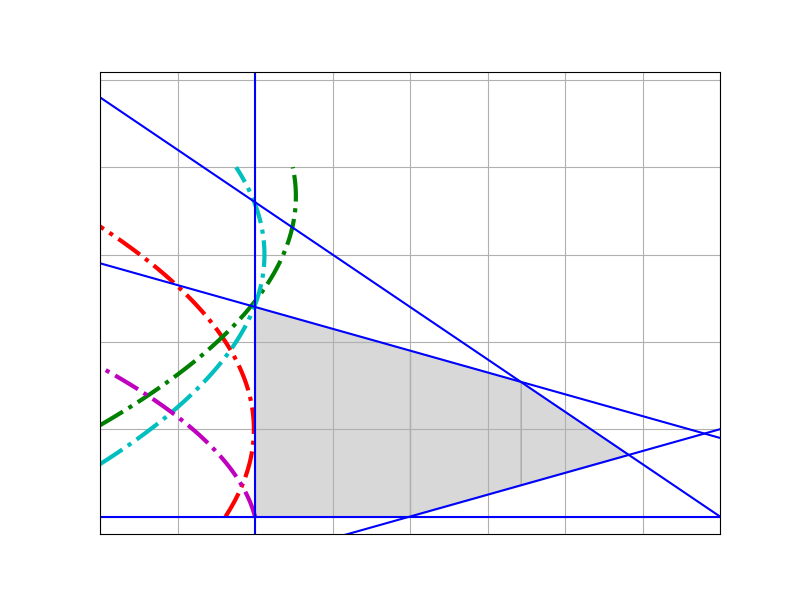}
    \caption{Solutions to the QP smoothing}
    \label{fig:QP}
  \end{subfigure}
  \vspace{0.5em}
    \caption{Schematic diagram showing the effect of QP smoothing. 
    (a) LP solutions and the corresponding isocost line for four cost vectors.
    The \textcolor{Green}{green}, \textcolor{cyan}{cyan} and \textcolor{red}{red} cost vectors result in the same solution, the top vertex, highlighting that a slight rotation of the isocost lines may not alter the LP solution. However, if the isocost lines rotate too much, for example, the \textcolor{violet}{violet} line, the solution suddenly shifts to a different vertex. (b) The isocost lines change after applying QP smoothing, and the solution is no longer restricted to a vertex. For example, the \textcolor{red}{red} vector results in a smooth change in the solution. However, even with smoothing, some cost vectors, like the \textcolor{cyan}{cyan} and \textcolor{Green}{green}, may still share the same solution.}
    \label{fig:QPSmoothing}
\end{figure*}

\paragraph{Differentiable Optimization by Smoothing.}
`Differentiable Optimization by Smoothing' is one approach to to circumvent this challenge. 
The aim of differentiable optimization is to represent an optimization problem as a differentiable mapping from its parameters to its solution.
Since for a COP, this mapping is \textbf{not} differentiable,
one prominent research direction in DFL involves smoothing the combinatorial optimization problem into a differentiable optimization problem.
We particularly focus on smoothing by regularization. There exists another from of smoothing---smoothing by perturbation, as proposed by \citet{DBB, FY2020, imle, NID}.
In this work, we focus on optimization problems with linear objective functions such as LPs and ILPs.
For an LP, the solution will always lie in one of the vertices of the LP simplex. 
So, the LP solution remains unchanged as long as the cost vector changes while staying within the corresponding normal cone \citep{boyd2004convex}. However, the solution will suddenly switch to a different vertex if the cost vector slightly moves outside the normal cone, as illustrated in Figure~\ref{fig:LP}.
Because the solution abruptly jumps between the vertices, the LP solution is not a differentiable function of the cost vector.

As explained in Section~\ref{sect:smoothing}, approaches under this category replace the original optimization problem with a `smoothed' version of the optimization problem, in which the solution can be expressed as a differentiable mapping of the parameter. 
For instance, if the original problem is an LP, it can be replaced with a QP by adding a quadratic regularizer to the objective of the LP.

The effect of QP smoothing is shown in Figure~\ref{fig:QP}.
The arrows represent the cost vectors, while the nonlinear curves indicate the iso-cost curves. After smoothing, the solution is not restricted to being at a vertex of the LP polyhedron.
In the `smoothed' problem, unlike the original LP, the solution do not change abruptly.
The solution either may not change or change smoothly with the change of the cost vector, as illustrated in Figure~\ref{fig:QP}. Consequently, $\solution (\predcost)$ becomes differentiable with respect to $\predcost$ and the solution, $\solution(\cost)$, can be represented as a differentiable function of the parameter $\cost$.
When the problem is an ILP, first LP, resulting from continuous relaxation is considered and then it is smoothed by adding quadratic regularizer.
Algorithm~\ref{alg:smoothing} explains this approach.

What we stress is that for a low smoothing strength (which ensures the solution after smoothing is not entire different from the LP solution), the QP solution might be same as the LP solution. Hence, two cost vectors can have the same solution. This is illustrated in Figure \ref{fig:QP}.
In this QP, 

DYS-Net \citep{mckenzie2024learning} provides an approximate solution to the quadratically regularized LP problem, where the computations are designed to be executed as standard neural network operations, enabling back-propagation through it.
To summarize, approaches in this category follow the training loop in Figure \ref{fig:dfl} but only after `smoothing' the optimization problem.
\begin{figure}
\centering
\begin{tikzpicture}[
  font=\rmfamily\footnotesize,
  every matrix/.style={ampersand replacement=\&,column sep=2cm,row sep=.6cm},
  source/.style={draw,thick, inner sep=.3cm},
  process/.style={draw,thick,rounded corners,inner sep=.3cm},
  sink/.style={source,fill=green!20},
  datastore/.style={draw,very thick,shape=datastore,inner sep=.3cm},
  dots/.style={gray,scale=2},
  to/.style={->,>=stealth',shorten >=1pt,semithick,font=\rmfamily\small},
  every node/.style={align=center}, 
  arrow/.style = {thick,-stealth}]

    \node[source] (z) at (0, 0) {$\feature$}; 
     \node[process](nn) at (2,0) { 
     \usebox{\mybox}};
     \node[process](LP) at (7,0) [minimum height=1cm]{$\solution (\predcost) = \argmin_{\decisionvar} \predcost^\top \decisionvar \; \;$ \\ $\text{s.t.} \;\;A \decisionvar = \mathbf{b};  \;\;
    \decisionvar \geq \bm{0}$}; 
     \node[process] (task) at (12,0) {$\regret(\solution (\predcost), \cost)$};
     \node[source](c) at (7.2,-2.5){$(\feature,\cost, \solution (\cost) )$}; 
     \path (nn) -- node[sloped] (hatc) {$\predcost$} (LP);
     \path (LP) -- node[sloped] (xhatc) {$\solution (\predcost)$} (task);
     \node   (backprop1) [below of=hatc] { \LARGE \textcolor{black}{ $\frac{d \regret}{d \predcost}$}};
     \node   (backprop2) [below of=xhatc] {\LARGE \textcolor{black}{$\frac{d \regret}{d  \solution (\predcost) }$}};
     \node (training) at (7.25,-2.5) [draw,dotted, thick,minimum width=3.5cm,minimum height=1.5cm] {};
     \node   (traintext) [below of=training] {Training Data };
     \node   (derivativetext) [below=0.05cm of LP] { \LARGE \textcolor{magenta}{$\frac{d \solution(\predcost)}{d \predcost}$} };

  \draw [arrow] (nn)--(hatc)--(LP);
  \draw [arrow] (LP)--(xhatc)--(task);
  \draw [arrow] (c) -|(task);
  \draw [arrow] (c) -|(z);
  \draw [arrow] (z) --(nn);
  \draw [arrow] (task.190) --(LP.352);
  \draw [arrow] (LP.192) --(nn.337);
\end{tikzpicture}
\caption{Decision-focused learning training loop.}
\vspace{0.5em}
\label{fig:dfl}
\end{figure}
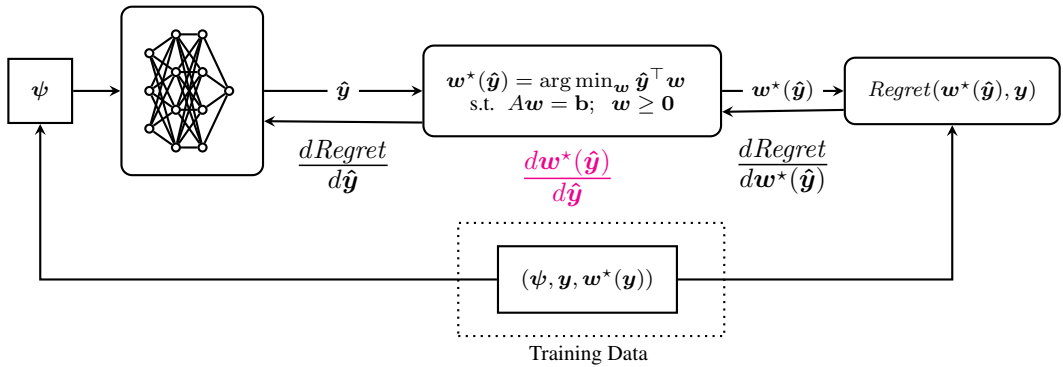 
\paragraph{Surrogate Losses for DFL.}
The primary goal of DFL is to minimize $\regret$. However, as explained earlier, $\regret$ cannot be minimized directly due to its non-differentiability. Techniques involving surrogate losses aim to address this challenge by identifying suitable surrogate loss functions and computing gradients or subgradients of these surrogate losses for optimization.
Figure \ref{fig:dfl_surrogate} depicts the training loop of DFL using surrogate loss functions. 
In this approach, $\regret(\solution (\predcost), \cost)$ is not explicitly computed. Instead, after predicting $\predcost$, a new cost vector $\bm{\Tilde{y}}$ is generated based on $\predcost$ and $\cost$, and the optimization problem is solved using this $\bm{\Tilde{y}}$. Subsequently, a surrogate loss is computed, using $\solution (\bm{\Tilde{y}})$ and $\solution (\cost)$, and its gradient,$\nabla$ (shown in pink) , is used for backpropagation.
We have explained this in terms of pseudocode using Algorithm \ref{alg:surrogate}.
For example, in the case of $\spoplus$, $\bm{\Tilde{\cost}} = 2\predcost - \cost$. 
As shown in \Eqref{eq:spoplus}
$\spoplus = $ $(2\predcost - \cost)^\top \solution (\cost)- (2\predcost - \cost)^\top \solution (2\predcost - \cost)$
 Then the gradient used for backpropagation is $\nabla = 2(\solution (\cost) -  \solution (2\predcost - \cost) )$.
On the other hand, in the case of $\SCE$, $\bm{\Tilde{\cost}} = \predcost$ and
$\SCE =
    \predcost^\top (\solution (\cost) - \solution (\predcost)) + \cost^\top ( \solution (\predcost) - \solution (\cost))$. So, in this case, the gradient for backpropagation is $\nabla = (\solution (\cost) - \solution (\predcost)) $.

\begin{figure}
\centering
\begin{tikzpicture}[
  font=\rmfamily\footnotesize,
  every matrix/.style={ampersand replacement=\&,column sep=2cm,row sep=.6cm},
  source/.style={draw,thick, inner sep=.3cm},
  process/.style={draw,thick,rounded corners,inner sep=.3cm},
  sink/.style={draw,thick,ellipse, minimum width=1pt,
    align=center},
  datastore/.style={draw,very thick,shape=datastore,inner sep=.3cm},
  dots/.style={gray,scale=2},
  to/.style={->,>=stealth',shorten >=1pt,semithick,font=\rmfamily\small},
  every node/.style={align=center}, 
  arrow/.style = {thick,-stealth}]

    \node[source] (z) at (0, 0) {$\feature$}; 
     \node[process](nn) at (2,0) { 
     \usebox{\mybox}};
     \node[process](LP) at (7.5,2) [minimum height=1cm]{$\solution (\bm{\Tilde{y}}) = \argmin_{\decisionvar} \bm{\Tilde{y}}^\top \decisionvar \; \;$ \\ $\text{s.t.} \;\;A \decisionvar = \mathbf{b};  \;\;
    \decisionvar \geq \bm{0}$}; 
     \node[sink] (mixer) at (4,0) {};

     \node[source] (transformed) at (4, 2) {$\bm{\Tilde{y}}$};
     \node[source] (grad) at (11,-0.5) {$\nabla $};
     \node[source](c) at (4,-2.5){$(\feature,\cost, \solution (\cost) )$}; 
     \path (nn) -- node[sloped] (hatc) {$\predcost$} (mixer);
     \node (training) at (4,-2.5) [draw,dotted, thick,minimum width=3.5cm,minimum height=1.5cm] {};
     \node   (traintext) [below of=training] {Training Data };
     \node   (derivativetext) [below=1.25cm of LP] {  \large \textcolor{magenta}{backpropgating $\nabla$} };

  \draw [arrow] (nn)--(hatc)-- (mixer);
  \draw [arrow] (c) -|(grad);
  \draw [arrow,dotted, magenta] (grad) --(nn.335);
  \draw [arrow] (z) --(nn);
   \draw [arrow] (c) --(mixer);
   \draw [arrow] (mixer) --(transformed);
  \draw [arrow] (transformed) --(LP);
  \draw [arrow] (c) -|(z);
  \draw [arrow] (LP) --(grad);
\end{tikzpicture}
\caption{Decision-focused learning using surrogate loss functions.}
\vspace{2em}
\label{fig:dfl_surrogate}
\end{figure}
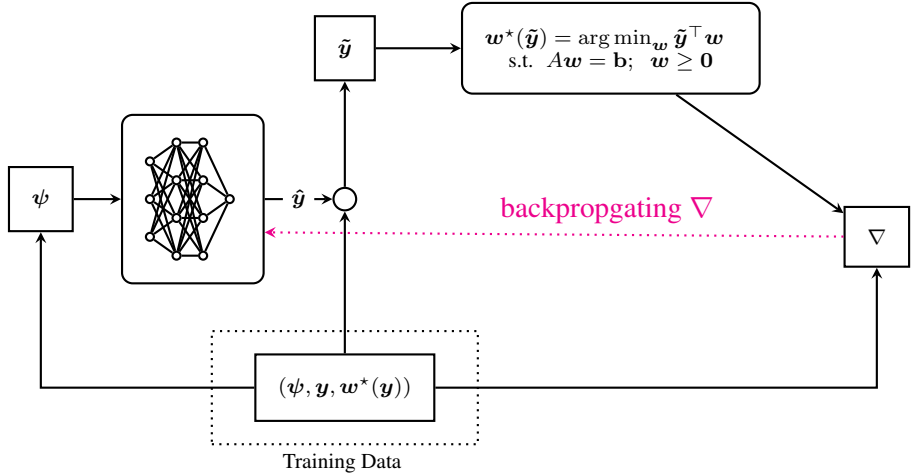 

\paragraph{Combining Surrogate Losses with Differentiable Optimization.}
The core idea proposed in this paper is to combine these two approaches. Specifically, the original optimization problem in Figure ~\ref{fig:dfl_surrogate} is replaced with a smoothed version, \textbf{allowing direct backpropagation through the smoothed problem instead of using $\nabla$}. We emphasis that this changes the gradient of the surrogate losses, i.e., Eq.~\ref{eq:SPO_fullgrad} and Eq.~\ref{eq:SCE_fullgrad} would reduce to SPO+ and SCE subgradients, respectively.
We explain this in Algorithm \ref{alg:smoothedsurrogate}.
\section {Proofs}
\subsection {Proof of Theorem \ref{label:pretheorem}}\label{proof_ppretheorem}

\begin{proof}
Let us consider the following optimization problem:
$$
\min_{\decisionvar \in \feas} \cost^\top \decisionvar 
$$
We assume $\solution (\cost)$ is the unique optimal solution, and therefore there is one unique $ \solution (\cost)$, so that,  ${\decisionvar^\prime}^\top  \cost  - \solution (\cost)^\top \cost \geq 0  \;  \forall \decisionvar^\prime \in \feas$. So, $\nabla_{\spoplus } $ would be zero only if the solution using $(2\predcost - \cost)$ is same as the true solution $ \solution (\cost)$.

\noindent Hence we can write,
\begin{align*}
    \nabla_{\spoplus }  = 0  \iff \solution(2\predcost - \cost)  = \solution ( \cost) &\iff (2\predcost - \cost)^\top \solution ( \cost) \leq (2\predcost - \cost)^\top \decisionvar^\prime  ; \; \; \forall \decisionvar^\prime \in \feas \nonumber \\
    & \iff 2\predcost ^\top (\decisionvar^\prime - \solution ( \cost) ) \geq \cost ^\top (\decisionvar^\prime - \solution ( \cost) ) \; \; \forall \decisionvar^\prime \in \feas \nonumber \\
    &\iff \predcost^\top \decisionvar^\prime - \predcost^\top \solution(\cost ) \geq \frac{1}{2} \bigg (\cost ^\top (\decisionvar^\prime - \solution ( \cost) ) \bigg) ; \; \; \forall \decisionvar^\prime \in \feas 
    \end{align*}

So, we can write:
\begin{align}
     \label{eq:spo_implies}
     \m{Y}_{ \mathit{SPO}^{+} } (\cost) = 
     \{ \predcost: \predcost^\top \decisionvar^\prime - \predcost^\top \solution(\cost ) \geq \frac{1}{2} \bigg (\cost ^\top (\decisionvar^\prime - \solution ( \cost) ) \bigg) ; \; \; \forall \decisionvar^\prime \in \feas  \}
\end{align}
As, $(\cost ^\top (\decisionvar^\prime - \solution ( \cost) ) ) \geq 0$, we can write
$$\predcost^\top \decisionvar^\prime - \predcost^\top \solution(\cost ) \geq 0 \quad \forall \;  \predcost \in \m{Y}_{ \mathit{SPO}^{+} } (\cost)$$
Hence, $\solution(\cost )$ is an optimal solution to any $\predcost \in \m{Y}_{ \mathit{SPO}^{+} } (\cost)$.

Similarly, for $\SCE$,
\begin{align*}
    \nabla_{\SCE } = 0 \iff \solution(\predcost )  = \solution ( \cost) &\iff  \predcost^\top  \solution(\cost )
    \leq \cost^\top \decisionvar^\prime; \; \; \forall \decisionvar^\prime \in \feas 
   \nonumber\\
    &\iff \predcost^\top \decisionvar^\prime - \predcost^\top \solution(\cost ) \geq 0 ; \; \; \forall \decisionvar^\prime \in \feas 
\end{align*}
Hence, we can write:
\begin{align}
     \label{eq:sce_implies}
     \m{Y}_{ \mathit{SCE} } (\cost) = 
     \{ \predcost: \predcost^\top \decisionvar^\prime - \predcost^\top \solution(\cost ) \geq 0 ; \; \; \forall \decisionvar^\prime \in \feas \}
\end{align}
This shows that $\solution(\cost )$ is an optimal solution to any $\predcost \in  \m{Y}_{ \mathit{SCE} } (\cost)$.
\end{proof}


\subsection {Proof of Theorem \ref{theorem:1}}\label{proof_theorem1}
As a preliminary step toward the proof, we establish the following corollaries:
\begin{corollary}
\label{corollary3}
For a given cost vector, $\cost \in \mathbb{R}^K$, 
    $\m{Y}_{ \mathit{SPO}^{+} } (\cost) \subseteq \m{Y}_{ \mathit{SCE} } (\cost)$.
\end{corollary}
\begin{proof}
    We showed in \eqref{eq:spo_implies} that
    $$\m{Y}_{ \mathit{SPO}^{+} } (\cost) = 
     \{ \predcost: \predcost^\top \decisionvar^\prime - \predcost^\top \solution(\cost ) \geq \frac{1}{2} \bigg (\cost ^\top (\decisionvar^\prime - \solution ( \cost) ) \bigg) ; \; \; \forall \decisionvar^\prime \in \feas  \}$$

We also proved in \eqref{eq:sce_implies} that
$$ \m{Y}_{ \mathit{SCE} } (\cost) = 
     \{ \predcost: \predcost^\top \decisionvar^\prime - \predcost^\top \solution(\cost ) \geq 0 ; \; \; \forall \decisionvar^\prime \in \feas \}$$
As, $\cost ^\top (\decisionvar^\prime - \solution ( \cost) )  \geq 0$, it is clear that  $\m{Y}_{ \mathit{SPO}^{+} } (\cost) \subseteq \m{Y}_{ \mathit{SCE} } (\cost)$.
\end{proof}
\begin{corollary}
\label{corollary4}
    For a given $\cost \in \mathbb{R}^K$, let us assume that there exists $\predcost_1$, $\predcost_2$ $\in \mathbb{R}^K$ and a feasible solution $\Tilde{\decisionvar}$, such that
    $$\predcost_1^\top \Tilde{\decisionvar} - \predcost_1^\top \solution(\cost ) \geq \frac{1}{2} \big (\cost ^\top (\Tilde{\decisionvar} - \solution ( \cost) ) \big)  $$ and
     $$0 \leq \predcost_2^\top \Tilde{\decisionvar} - \predcost_2^\top \solution(\cost ) < \frac{1}{2} \big (\cost ^\top (\Tilde{\decisionvar} - \solution ( \cost) ) \big)  $$
    Then the following statements hold:
\begin{enumerate}
\item If for any $\bm{\Delta} \in \mathbb{R}^K$, 
$\left( \predcost_1 + \bm{\Delta} \right)^\top \solution(\cost) 
    > 
   \left( \predcost_1 + \bm{\Delta} \right)^\top \Tilde{\decisionvar} \implies 
   \left( \predcost_2 + \bm{\Delta} \right)^\top \solution(\cost) 
    > 
   \left( \predcost_2 + \bm{\Delta} \right)^\top \Tilde{\decisionvar}
   $.
    \item There exists $\bm{\Delta} \in \mathbb{R}^K$, such that $\left( \predcost_2 + \bm{\Delta} \right)^\top \solution(\cost) 
    > 
   \left( \predcost_2 + \bm{\Delta} \right)^\top \Tilde{\decisionvar}$, but 
   $\left( \predcost_1 + \bm{\Delta} \right)^\top \solution(\cost) 
    \leq 
   \left( \predcost_1 + \bm{\Delta} \right)^\top \Tilde{\decisionvar}$.
\end{enumerate}
\end{corollary}
\begin{proof}

    \begin{enumerate}
        \item
        \begin{align*}
        \left( \predcost_1 + \bm{\Delta} \right)^\top \solution(\cost) 
    > 
   \left( \predcost_1 + \bm{\Delta} \right)^\top \Tilde{\decisionvar} \implies
   \bm{\Delta}^\top (\solution(\cost) - \Tilde{\decisionvar}) >  \predcost_1^\top ( \Tilde{\decisionvar}  - \solution(\cost) )) \geq \frac{1}{2} \cost^\top ( \Tilde{\decisionvar}  - \solution(\cost) )) > \predcost_2^\top ( \Tilde{\decisionvar}  - \solution(\cost) ))
   \end{align*}
   After rearranging, we can write:
   \begin{align*}
     \left( \predcost_1 + \bm{\Delta} \right)^\top \solution(\cost) 
    > 
   \left( \predcost_1 + \bm{\Delta} \right)^\top \Tilde{\decisionvar} \implies  
   \left( \predcost_2 + \bm{\Delta} \right)^\top \solution(\cost) 
    > 
   \left( \predcost_2 + \bm{\Delta} \right)^\top \Tilde{\decisionvar}
   \end{align*}
   \item 
   We begin by writing the assumption
   $$\predcost_2^\top \Tilde{\decisionvar} - \predcost_2^\top \solution(\cost ) < \frac{1}{2} \big (\cost ^\top (\Tilde{\decisionvar} - \solution ( \cost) ) \big)$$
   This implies the following:
   \begin{align*}
       \left(\predcost_2 - \frac{1}{2} \cost \right)^\top \Tilde{\decisionvar} < \left(\predcost_2 - \frac{1}{2} \cost \right)^\top  \solution (\cost)
   \end{align*}
Be defining, $ \bm{\Delta} = -\frac{1}{2} \cost $, we can write
 $\left( \predcost_2 + \bm{\Delta} \right)^\top \solution(\cost) 
    > 
   \left( \predcost_2 + \bm{\Delta} \right)^\top \Tilde{\decisionvar}$.
   
   But,  $$\predcost_1^\top \Tilde{\decisionvar} - \predcost_1^\top \solution(\cost ) \geq \frac{1}{2} \big (\cost ^\top (\Tilde{\decisionvar} - \solution ( \cost) ) \big)  $$.
   
   Hence, $\left( \predcost_1 + \bm{\Delta} \right)^\top \solution(\cost) 
    \leq 
   \left( \predcost_1 + \bm{\Delta} \right)^\top \Tilde{\decisionvar}$.
    \end{enumerate}
\end{proof}
\paragraph{Proof of Theorem \ref{theorem:1}.}
\begin{proof}
We will use $\{ \bm{\Delta} \}_{ \mathit{SCE}(\cost) }$ to denote the set of all $\bm{\Delta}$ such that $(\predcost + \bm{\Delta})^\top \solution (\cost) > (\predcost + \bm{\Delta})^\top \solution (\predcost + \bm{\Delta})$ for any $\predcost \in \m{Y}_{ \mathit{SCE} } (\cost)$.
Similarly, $\{ \bm{\Delta} \}_{\mathit{SPO}^{+}(\cost) }$ denotes the set of all $\bm{\Delta}$ such that $(\predcost + \bm{\Delta})^\top \solution (\cost) > (\predcost + \bm{\Delta})^\top \solution (\predcost + \bm{\Delta})$ for any $\predcost \in \m{Y}_{  \mathit{SPO}^{+} } (\cost)$.
This allows us to rewrite 
$$\Gamma_{\mathit{SCE}} = \min_{\bm{\Delta} \in \{ \bm{\Delta} \}_{ \mathit{SCE}(\cost) } } \|\bm{\Delta}\|_2$$
$$\Gamma_{\mathit{SPO}^{+}} = \min_{\bm{\Delta} \in \{ \bm{\Delta} \}_{ \mathit{SPO}^{+}(\cost) } } \|\bm{\Delta}\|_2$$

We have proved in Corollary~\ref{corollary3} that $\m{Y}_{ \mathit{SPO}^{+} } (\cost) \subseteq \m{Y}_{ \mathit{SCE} } (\cost)$.
It is straightforward that if $\{ \bm{\Delta} \}_{ \mathit{SCE}(\cost) }$ and $\{ \bm{\Delta} \}_{\mathit{SPO}^{+}(\cost) }$ are the same set if $\m{Y}_{ \mathit{SPO}^{+} } (\cost)$ and $ \m{Y}_{ \mathit{SCE} } (\cost)$ are identical. Hence, in this case, it follows trivially that $\Gamma_{\mathit{SCE}} =
    \Gamma_{\mathit{SPO}^{+}} $.
    
Let us consider the case, when $\m{Y}_{ \mathit{SPO}^{+} } (\cost) \subset\m{Y}_{ \mathit{SCE} } (\cost)$.
This implies there exist a $\predcost_2 \in \m{Y}_{ \mathit{SCE} } (\cost) \setminus \m{Y}_{ \mathit{SPO}^{+} } (\cost)$. 
The second statement in Corollary~\ref{corollary4} shows that there exists a perturbation vector, $\bm {\Delta}$, such that
$\solution(\predcost_2 + \bm {\Delta} ) \neq \solution (\cost)$ for some $\predcost_2 \in \m{Y}_{ \mathit{SCE} } (\cost) \setminus \m{Y}_{ \mathit{SPO}^{+} } (\cost)$, while for all $\predcost_1 \in \m{Y}_{ \mathit{SPO}^{+} } (\cost)$, $\solution(\predcost_1 + \bm {\Delta} ) = \solution (\cost)$.
Hence, $\bm {\Delta} \in \{ \bm{\Delta} \}_{ \mathit{SCE}(\cost) }  $, but $\bm {\Delta} \notin \{ \bm{\Delta} \}_{ \mathit{SPO}^{+}(\cost) }  $.

Now, if for any 
$\predcost_1 \in \m{Y}_{ \mathit{SPO}^{+} } (\cost)$, $\solution(\predcost_1 + \bm {\Delta}_1 ) \neq \solution (\cost)$, then $\bm {\Delta}_1 \in \{ \bm{\Delta} \}_{ \mathit{SPO}^{+}(\cost) } $.
Moreover, as $\m{Y}_{ \mathit{SPO}^{+} } (\cost) \subset  \m{Y}_{ \mathit{SCE} } (\cost)$, $\predcost_1 \in \m{Y}_{ \mathit{SCE} } (\cost)$.
Hence,  $\bm {\Delta}_1 \in \{ \bm{\Delta} \}_{ \mathit{SCE}(\cost) } $. 
To summarize, $\bm {\Delta} \in \{ \bm{\Delta} \}_{ \mathit{SPO}^{+}(\cost) } \implies$ $\bm {\Delta} \in \{ \bm{\Delta} \}_{ \mathit{SCE}(\cost) }  $, but not the reverse.
So, $$  \{ \bm{\Delta} \}_{\mathit{SPO}^{+}(\cost) }  \subset \{ \bm{\Delta} \}_{ \mathit{SCE}(\cost) } $$


As, $\{ \bm{\Delta} \}_{\mathit{SPO}^{+}(\cost) }\subset \{ \bm{\Delta} \}_{ \mathit{SCE}(\cost) }$,
$$\Gamma_{\mathit{SPO}^{+}} = \min_{\bm{\Delta} \in \{ \bm{\Delta} \}_{ \mathit{SPO}^{+}(\cost) } } \|\bm{\Delta}\|_2  \geq \Gamma_{\mathit{SCE}} = \min_{\bm{\Delta} \in \{ \bm{\Delta} \}_{ \mathit{SCE}(\cost) } } \|\bm{\Delta}\|_2 $$

\end{proof}

\subsection{Proof of Proposition \ref{prop:1}}\label{proof_prop1}
\begin{proof}
\begin{enumerate}
    \item Following the definition of $\SCE$,
    \begin{align*}
    \SCE(\solution (\predcost), \cost)&= (\predcost - \cost)^\top \solution (\cost)-  (\predcost - \cost)^\top \solution (\predcost) \\&=
    \predcost^\top (\solution (\cost) - \solution (\predcost)) + \cost^\top ( \solution (\predcost) - \solution (\cost))
    \end{align*}
    $\predcost^\top (\solution (\cost) - \solution (\predcost)) \geq 0$, because $\solution (\predcost)$ is the optimal solution to $\predcost$.
    In a similar way, $\cost^\top ( \solution (\predcost) - \solution (\cost)) \geq 0$. Hence, $\SCE(\solution (\predcost), \cost) \geq 0$.
    \item We will prove the claim by contradiction. Assume that $\m{L}_{SCE}(\solution (\predcost), \cost)=0$ but $\regret(\solution (\predcost), \cost)= \cost^\top ( \solution (\predcost) - \solution (\cost)) > 0$ . 
    This is possible if $ \solution (\predcost) \neq \solution (\cost)$.
    As the solution to $\predcost$ is different from $\solution (\cost)$, the singleton assumption implies that $\exists \, \decisionvar^\prime \in  \feas \setminus \{\solution (\cost) \}  : \,  \predcost^\top \decisionvar^\prime < \predcost^\top \solution (\cost) $. 
	In this case, we have:
	\begin{align*}
		 &\predcost^\top \solution (\cost) - \predcost^\top \decisionvar^\prime >0 \\
		&\Rightarrow  ( \predcost^\top \solution (\cost) - \predcost^\top \decisionvar^\prime ) + 
		(  \cost^\top \decisionvar^\prime - \cost^\top \solution (\cost) ) >  (  \cost^\top \decisionvar^\prime - \cost^\top \solution (\cost) ) \geq 0\\
		&\Rightarrow  (\predcost - \cost)^\top \solution -  (\predcost - \cost)^\top \solution (\predcost) > 0 
	\end{align*}
\noindent	In the second line, $\cost^\top \decisionvar^\prime - \cost^\top \solution (\cost) $ is added in both sides and this term is nonnegative as $\solution (\cost)$ is the optimal solution to $\cost$. 
	This implies $\m{L}_{SCE}(\solution (\predcost), \cost)>0$ and we arrive at a contradiction. Thus we prove that 
    $\m{L}_{SCE}(\solution (\predcost), \cost)=0$ $\implies$ $\regret(\solution (\predcost), \cost)=0$.
    
    Next, assume  $\regret(\solution (\predcost), \cost)=0$. This implies that $\cost^\top  \solution (\predcost) = \cost^\top \solution (\cost)$.
    This can only be true if $\solution (\predcost) = \solution (\cost)$ because of the singleton assumption. 
    Hence, $\SCE(\solution (\predcost), \cost)$ $= (\predcost - \cost)^\top ( \solution (\cost) - \solution (\predcost) ) =0 $.
    \end{enumerate}
\end{proof}
\subsection{Generalization Bounds for SCE Loss} \label{generalization_bounds}

We show generalization bounds for SCE loss similar to the bounds shown for true regret by \cite{el2019generalization}.  For notational brevity, we first define SCE in terms of the predicted and true parameters, i.e., \begin{equation*}
    l_{SCE}(\predcost,\cost) = \m{L}_{SCE} (\solution (\predcost), \cost) = (\predcost - \cost)^\top (\solution - \solution (\predcost))
\end{equation*}  
where $\predcost = \MLnoomega(\feature)$ is the predicted cost using the predcitive model $\MLnoomega$. We can also define
\begin{equation*}
    R_{SCE}(\MLnoomega) = \mathbb{E}[l_{SCE}(\MLnoomega(\feature),y)] \text{ and }  \widehat{R}_{SCE}(\MLnoomega) = \frac{1}{N} \sum_{i=1}^N l_{SCE}(\MLnoomega(\feature_i),y_i)
\end{equation*}
as the true and empirical risk for a given sample $\{ (\feature_i, \cost_i) \}_{i=1}^N$ for SCE loss, respectively. 

In order to show generalization bounds for SCE loss, we need to define the Rademacher complexity of a set of functions $\mathcal{H}$ with $l_{SCE}$. The sample Rademacher complexity for a given sample $\{ (\feature_i, \cost_i) \}_{i=1}^N$ is given by
\begin{equation*}
\widehat{\mathfrak{R}}_{SCE}^N(\mathcal{H}) = \mathbb{E}_{\sigma} \left[ \sup_{\MLnoomega \in \mathcal{H}} \frac{1}{N} \sum_{i=1}^N\sigma_i l_{SCE}(\MLnoomega(\feature_i),\cost_i)\right]
\end{equation*}
where $\sigma_1, \sigma_2, \ldots, \sigma_N$ are i.i.d. random variables with $\mathbb{P}(\sigma_i = 1) = 1/2$ and $\mathbb{P}(\sigma_i = -1) = 1/2$ for $i=1,2,\ldots,N$. The expected  Rademacher complexity is defined as $\mathfrak{R}_{SCE}^N(\mathcal{H}) = \mathbb{E}[\widehat{\mathfrak{R}}_{SCE}^N(\mathcal{H})]$ where the expectation is with respect to the i.i.d. samples of size $N$ from the true distribution. 

Assume that the set of all feasible solutions $\feas = \{ \decisionvar: A \decisionvar = \mathbf{b};  \;\;
    \decisionvar \geq \bm{0} \}$ is bounded, i.e., there exists $D$ such that $D = max_{\decisionvar,\decisionvar' \in \feas} \| \decisionvar - \decisionvar'\|$. Also assume that the set of all cost vectors is $\mathcal{Y}$ such that $\mathcal{Y} \subseteq \{\cost : \|\cost\| \leq 1\}$. Note that, since we consider linear objective functions, this assumption is not restrictive, as $\cost'$ with $\|\cost'\| > 1$ can be replaced by $\cost = \cost'/\| \cost'\|$ without changing the optimal solution and ensuring  $\|\cost\| = 1$. 

The following proposition shows the generalization bound for SCE loss. 
\begin{proposition}
    Let $\mathcal{H}$ be a set of functions from the set of all features to $\{\cost: \|\cost\| \leq 1\}$. Then for any $\delta > 0$,
    \begin{equation}\label{eq:for_each_ML}
        R_{SCE}(\MLnoomega) - \widehat{R}_{SCE}(\MLnoomega) \leq 2\mathfrak{R}_{SCE}^N(\mathcal{H}) + 2D\sqrt{\frac{\log(1/\delta)}{2N}}
    \end{equation}
    
    holds for all $\MLnoomega \in \mathcal{H}$ with probability at least $1-\delta$ for the sample $\{ (\feature_i, \cost_i) \}_{i=1}^N$  from the joint distribution of features and parameters.  If $\widehat{\MLnoomega}_n$ is a  minimizer of the emprical risk $\widehat{R}_{SCE}$, then the inequality 
    \begin{equation}\label{eq:minimizer_ML}
        R_{SCE}(\widehat{\MLnoomega}_n) - \min_{\MLnoomega \in \mathcal{H}} R_{SCE}(\MLnoomega) \leq 2\mathfrak{R}_{SCE}^N(\mathcal{H}) + 4D\sqrt{\frac{\log(2/\delta)}{2N}}
    \end{equation}
    also holds probability at least $1-\delta$. 
\end{proposition}

\begin{proof}
    The SCE loss is bounded for all $\cost, \cost' \in \mathcal{Y}$ since $l_{SCE}(\predcost,\cost) = (\predcost - \cost)^\top (\solution - \solution (\predcost))  \leq \| \predcost - \cost\| \| \solution - \solution (\predcost)\| \leq 2D$ where the first inequality is due to Cauchy–Schwarz and the second inequality is due to our assumptions on the hypothesis class and the feasible region. Then, inequality \ref{eq:for_each_ML} follows directly from the classical generalization bound as shown in \cite{bartlett2002rademacher}. 

    The extension of inequality \ref{eq:for_each_ML} to inequality \ref{eq:minimizer_ML} is shown in the proof of Corollary 1 in  \cite{el2019generalization} using Hoeffding's inequality. 
\end{proof}

\section{Computational Experiments Demonstrating Zero Gradient }
\label{Appendix_simu}
In Section~\ref{section:surrogateloss}, we made the case for minimizing surrogate loss such as $\SCE$ instead of $\regret$.
Our main argument is for a relatively low value of smoothing parameter $\mu$,
$\regret$ will have zero gradient.
However, $\SCE$ will not have this problem.
We provided two illustrations considering small-scale optimization problems. 
Here, we justify this with higher-dimensional optimization problems.
We consider Top-1 selection problem with different number of items $M$.
\begin{equation}
    \label{eq:TopkSimu}
     \max_{\decisionvar \in \{0,1\} } \cost^\top \decisionvar \;\;\text{s.t.}\; \decisionvar^\top \mathbf{1}\leq 1
\end{equation}
\noindent $\cost =[y_1, \ldots, y_M]  \in \mathbb{R}^M$ is the vector denoting value of all the items and $\decisionvar = [v_1, \ldots, v_M]$ is the vector decision variables.
To replicate the setup of a PtO problem, we solve the optimization problem with $\predcost$.
Let us assume $y_i, \hat{y}_i \geq 0$.

Before, solving the problem with simulation, we will show one interesting aspect of this problem.
Note that when $\mu>0$, the following relaxed optimization problem is solved:
\begin{equation}
    \label{eq:TopkSimuwithmu}
     \max_{\decisionvar } \cost^\top \decisionvar - \frac{\mu}{2} ||\decisionvar||^2 \;\;\text{s.t.}\; \decisionvar^\top \mathbf{1}\leq 1 ;\; \; \decisionvar \geq 0 
\end{equation}
We point out that the solution to the unconstrained optimization problem is $v_i^\star = \frac{y_i}{\mu} >0$.

The augmented Lagrangian of \Eqref{eq:TopkSimuwithmu} is
\begin{equation}
    \mathbb{L} = \cost^\top \decisionvar - \frac{\mu}{2} ||\decisionvar||^2 + \lambda (1 - \decisionvar^\top \mathbf{1} ) + \bm{\sigma}^\top \decisionvar
\end{equation}
\noindent where $\lambda$ and $\bm{\sigma} = [\sigma_1, \ldots, \sigma_M]$ are dual variables. 
By differentiating $ \mathbb{L}$ with respect to $v_i$, we obtain one condition of optimality, which is the following:
\begin{align}
    y_i - \mu v_i - \lambda + \sigma_i = 0 
   \implies v_i = \frac{y_i - \lambda + \sigma_i}{\mu}
\end{align}
Without any loss of generality, let $y^{(1)} \geq y^{(2)} \geq \ldots y^{(M)}$. (In the d)
As, solution to the constrained optimization problem is $v_i >0$, $y^{(1)}$ will definitely be greater than zero.
Hence, $\sigma_i=0$ because of strict complementarity.
So, we can write $v^{(1)} - v^{(k)} = \frac{ y^{(1)} - y^{(k)} -\sigma^{(k)} }{\mu}$.
As, $v^{(1)} - v^{(k)} \leq 1$, we can write:
\begin{align}
    \frac{ y^{(1)} - y^{(k)} -\sigma^{(k)} }{\mu} \leq 1 \implies \mu \geq y^{(1)} - y^{(k)} -\sigma^{(k)}
\end{align}
So,
\begin{align}
\label{eq:simulation_inference}
y^{(1)} - y^{(k)} > \mu \implies \sigma^{(k)} > 0 \implies  y^{(k)}=0
\end{align}
This suggest that if $y^{(k) } < y^{(1)} - \mu $, only $v^{(1)}=1$ and all other decision variables will be zero in the optimal solution.

To generate the ground truth $\cost$, we randomly select $M$ integers without replacement from the set ${1, \dots, M}$. The predicted costs, $\predcost$, are generated by considering a different sample from the same set. As a result, $\cost$ and $\predcost$ contain the same numbers but in different permutations. It is important to note that all elements in both vectors are positive integer values.
We compute the solution to the optimization problem for $\cost$ and $\predcost$.
We solve the optimization problem with $\predcost$ using a `smoothed' optimization layer---\textsl{CvxpyLayer}.
In order to compare the gradients of $\regret$ and $\SCE$.
We compute the gradients of both the losses for multiple values of $M$ and $\mu$. For each configuration of $M$ and $\mu$, we run 20 simulations.
\begin{table}[]
    \centering
\begin{tabular}{lrrrrrrrr}
\toprule
& \multicolumn{6}{c}{M}\\  
\cmidrule{2-7}
$\mu$ & 5 & 10 & 20 & 40 & 80 & 100   \\
\midrule
0.100& 0.000& 0.000& 0.000& 0.000& 0.000& 0.000\\
0.500& 0.000& 0.000& 0.000& 0.000& 0.000& 0.000\\
0.990& 0.000& 0.000& 0.000& 0.000& 0.000& 0.001\\
1.050& 0.089& 0.089& 0.089& 0.089& 0.089& 0.089\\
1.500& 0.465& 0.466& 0.465& 0.465& 0.465& 0.464\\
2.000& 0.622& 0.622& 0.622& 0.622& 0.622& 0.622\\
5.000& 1.165& 1.165& 1.165& 1.165& 1.165& 1.165\\
\bottomrule
\end{tabular}
    \caption{We tabulate average Manhattan distance  between the solution of the `smoothed' problem and the solution of the original problem for different values of $M$ and $\mu$.}
    \label{tab:smilate_table}
\end{table}

Note that $y^{(1)} > y^{(2)} > \ldots y^{(M)}$ because of the way we created the dataset. Moreover, as all values in $\predcost$ and $\cost $ are integer, \Eqref{eq:simulation_inference} suggests if $\mu < 1$, the solution to the relaxed problem (\eqref{eq:TopkSimuwithmu}) will be binary.
So, the discussion in Section~\ref{section:surrogateloss} suggests that slight change of the cost parameter would not change the solution and hence the zero gradient problem would appear while differentiating $\regret$.

In Figure \ref{fig:SimulationGradient}, we plotted the average absolute values of the gradients of the two losses--- $\SCE$ and $\regret$.
As we hypothesized the gradient turns zero whenever $\regret$ is minimized with $\mu <1$.
It is true that for $\mu > 1$, $\regret $ have non-zero gradient. However, higher values of $\mu $ turns solution to the `smoothed' problem very different from the solution to the original problem.
We show this in Table \ref{tab:smilate_table} by displaying the average Manhattan distance between solutions of the true and `smoothed' problem for same $\predcost$.

We also highlight that, for the same values of $\mu$, the average Manhattan distances remain same across different $M$.
Examining the results of the simulations, we observed that the solution to the smoothed problem is fractional. For example, when $\mu=2$, the solution includes two non-zero values--- 0.77 and 0.23. Typically, the value 0.77 appears in the position corresponding to the highest value in $\predcost$, i.e., where there is a 1 in solution vector. As a result, the Manhattan distance becomes (1-0.77)+0.23 = 0.46. Interestingly, these values remain unchanged across different values of $M$. Therefore, the Manhattan distance remains constant as long as $\mu$ does not change.
\begin{figure}
    \centering
\includegraphics[width=1\linewidth]{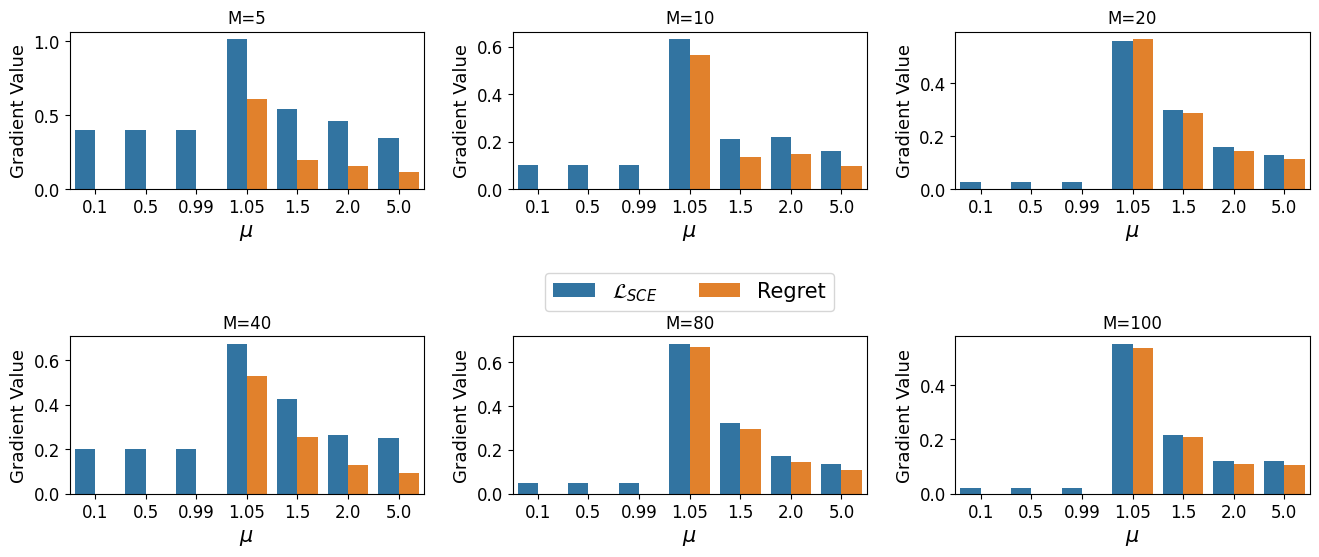}
    \caption{Results of Computational Simulation}
    \label{fig:SimulationGradient}
\end{figure}

\section{Demonstration of Learning with $\SCE$ versus Regret}
\label{Appendix:knapsack_demo}

We further illustrate this with a simple fractional knapsack problem, which is an LP.
Let us consider that we have two items and space for only one item.
This can be formulated as a minimization problem:\begin{equation*}
    \min\; \;  -y_1 v_1 - y_2 v_2 \; \;  \text{s.t.} \; \; v_1 + v_2 \leq 1; \; \; v_1, v_2 \geq 0
\end{equation*}
Let us assume the true values of $y_1$ and $y_2$ are $(0.8, 0.4)$. The corresponding solution is $(v_1,v_2)= (1,0)$.
The grey region in Figure \ref{fig:knpsackdemo} corresponds to any predictions satisfying $\hat{y}_1 > \hat{y}_2$. Such predictions will induce the true solution, resulting in zero regret.
Further assume that the initial predictions are 
$(\hat{y}_1, \hat{y}_2) = (0.1 , 0.01)$.
We show the progression of predictions by epochs when regret and SCE are used as training loss, using the smoothed optimization problem with \textcolor{blue}{blue} and \textcolor{Green}{green} lines, respectively in Figure \ref{fig:knpsackdemo}.
The predictions does not change with training epochs when regret is used as the loss because the derivatives of regret with respect to $\hat{y}_1$ and $\hat{y}_2$ are zero.
On the other hand, when $\m{L}_{SCE}$ is used as the loss, $(\hat{y}_1, \hat{y}_2)$ gradually move from the white region to the grey region, eventually resulting in zero regret.
Note that increasing the strength of smoothing may provide non-zero gradient across the space.
But this will entirely alter the optimization problem's solution.
For instance, in this knapsack example, high values of $\mu$ would make both $v_1$ and $v_2$ close to zero.

\begin{figure}
    \centering
\includegraphics[width=0.45\linewidth]{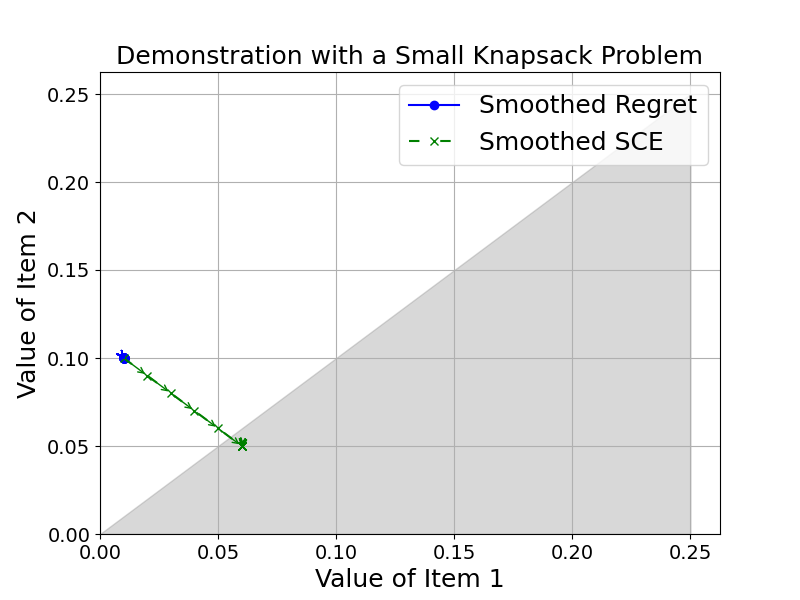}
    \caption{Progression of predictions by epochs when the smoothed regret and SCE are used as training losses. }
     \label{fig:knpsackdemo}
     \vspace{2em}
\end{figure}

\section{Description of The Optimization Problem and Data Generation Process}
\label{appendix:ILP_formulation}
In this section, we first describe the optimization problems along with their formulations, followed by details of the data generation process and the ML models.
\subsection{Description of the Optimization Problems}\label{appendix:optimization_formulation}
\paragraph{Shortest Path Problem.}
It is a shortest path problem on a $k \times k$ grid, with the objective of going from the southwest corner of the grid to the northeast corner where the edges can go either north or east. This grid consists of $k^2$ nodes and $2\times k \times (k-1)$ edges. 
Let, $y_{ij}$ is the cost of going from node $i$ to node $j$ and the decision variable $w_{ij}$ takes the value 1 if and only if the edge from node $i$ to node $j$ is traversed.
Then, the shortest path problem from going to node $s$ to node $t$ can be formulated as an LP problem in the following form:
\begin{subequations}
    \begin{align}
        \min_{w_{ij}} & \sum_{(i,j) \in \m{E}} y_{ij} w_{ij} \\
    \text{s.t.}\; & \nonumber \\
    & \sum_{(i,j)  \in \m{E} } w_{ij} - \sum_{(k,i)  \in \m{E} } w_{ki} =
    \begin{cases}
        1 \; & \text{if } \; i=s \\
        -1 \; & \text{if } \; i=t \\
         0 \; & \text{otherwise} \\
    \end{cases}\\
    & w_{ij}  \in \mathbb{R}^+
    \end{align}
\end{subequations}
\paragraph{Knapsack Problem.}
In a knapsack problem the goal of the optimization problem is to choose a maximal value subset from a given set of items, subject to some capacity constraints.
Let the set contains $\mathit{N_{items}}$ number of items and the value of each item is $y_i$.
The solution must satisfy capacity constraints in multiple dimensions. Let $C_j$ is the capacity in dimension $j$ and $\phi_{(i,j)}$ is the weight of item $i$ in dimension $j$.
This optimization can be modeled as an integer linear programming (ILP) as follows:
\begin{subequations}
    \label{eq:knapsackformulation}
    \begin{align}
   \min_{w_{1}, \ldots, w_{\mathit{N_{items}}}} & \sum_{i=1}^\mathit{N_{items}} (-y_i) w_i \\
   \text{s.t.}\; & \nonumber \\
    & \sum_{i=1}^\mathit{N_{items}} \phi_{(i,j)} w_i \leq C_j \; ; \; \forall j \\
    & w_i \in \{0,1\}
   \end{align}
\end{subequations}

 The Top-K selection can be viewed as a special case of the knapsack problem. In the Top-K, there is only one dimension and the weight of each item is $1$ and the capacity, $C=K$.
\paragraph{Capacitated facility location problem.}
Provided a set of locations for opening facilities, $F$, and a set of customers, $C$, the objective is to minimize the total operation cost while satisfying the customer demands of a single product.
Each customer \( c \in C \) has a demand \( \mathrm{D}_c \), and each facility \( f \in F \) has a capacity \( \mathrm{Cap}_f \) and incurs a fixed opening cost \( \mathrm{FC}_f \). 
A transportation cost is incurred to serve customer \( c \) from facility \( f \). 
The two constraints are to ensure that each customer's demand is fully met and ensuring the total demand assigned to a facility does not exceed its capacity.
The decision variables are:
\begin{itemize}
    \item \( w_{c,f} \in [0,1] \): the fraction of client \( c \)'s demand assigned to facility \( f \),
    \item \( W_f \in \{0,1\} \): a binary variable indicating whether facility \( f \) is open.
\end{itemize}
It is a an MILP, as it has both integer and continuous variables. The formulation of the problem is the following:
\begin{subequations}
\begin{align}
    \min_{w_{c,f},W_f } \quad & \sum_{c \in C} \sum_{f \in F} w_{c,f}  y_{c,f} \mathrm{D}_c + \sum_{f \in F} \mathrm{FC}_f \cdot W_f \label{eq:objective} \\
    \text{s.t.} \quad & \sum_{f \in F} w_{c,f} = 1, && \forall c \in C, \label{eq:c1} \\
    & \sum_{c \in C} \mathrm{D}_c \cdot w_{c,f} \leq \mathrm{Cap}_f \cdot W_f, && \forall f \in F, \label{eq:c2} \\
    & w_{c,f} \geq 0, \quad W_f \in \{0,1\}, && \forall c \in C, \forall f \in F. \label{eq:c3}
\end{align}
\end{subequations}
In the PtO version, the transportation costs are unknown, all other parameters are known.
\subsection{Description of the Data Generation Process}\label{appendix:data}
We use PyEPO library~\citep{pyepo} to generate training, validation and test datasets.
Each dataset consists of $\{ (\feature_i, \cost_i) \}_{i=1}^N$, which are generated synthetically.
The feature vectors are sampled from a multivariate Gaussian distribution with zero mean and unit variance, i.e., $\feature_i \sim \mathbf{N} (0, I_p)$, where p is the dimension of $\feature_i$.
To generate the cost vector, first a matrix $B \in \mathbb{R}^{K \times p} $ is generated, which represents the true underlying model, unknown to the modeler.
Each element in the cost vector $y_{i,j}$ is then generated according to the following formula:
\begin{equation}
    \cost_{ij} = \bigg [ \frac{1}{3.5^{\text{Deg}}}\bigg(\frac{1}{\sqrt{p}} \big(B \feature_i  \big) +3  \bigg)^{\text{Deg}  } +1 \bigg]\xi_i^j
\end{equation}
The \emph{Deg} is `model misspecification' parameter. This is because a linear model is used as a predictive model in the experiment and a higher value of \emph{Deg} indicates the predictive model deviates more from the true underlying model and larger the prediction errors. $\xi_i^j$ is a multiplicative noise term sampled randomly from the uniform distribution $\xi_i^j \sim U[1- w, 1+ w ]$. $w$ is a noise-half width parameter, which is less than 1. 
Higher values of $w$ indicate a greater degree of noise perturbation.
We set \emph{Deg} to 6 and $w$ to 0.5 in all our experiments.
\section{Retrieving only the Active Constraints on the Knapsack Problem}
\label{Appendix:Cave}

Consider the two-dimensional knapsack example in Figure \ref{fig:cave}. The capacity constraint is given as $3 v_1 + 3 v_2 \leq 5$ .  If the objective vector $\cost$ lies within the union of the yellow and red cones, then the feasible solution (1,0) is optimal for the problem with the integrality constraint. 
So, the true normal cone is the union of the yellow and red cones.
Note that the constraint $3 v_1 + 3 v_2 \leq 5$ is not active, although it plays a key role in choosing the solution; in the absence of this constraint, the solution would be $(1,1)$.
In this case, the only active constraints are $v_1=1$ and $v_2=0$.
As the CaVE approach stores only these two constraints, 
the yellow cone is considered as the optimality cone. This example shows that the mismatch between the cone of optimality of the integer knapsack and its relaxation can be non-trivial (the red cone in Figure \ref{fig:cave}).
This attributes to the poor performance of the CaVE approach in the Knapsack problem.
\begin{figure}
\begin{center}
     \includegraphics[width=0.45\textwidth]{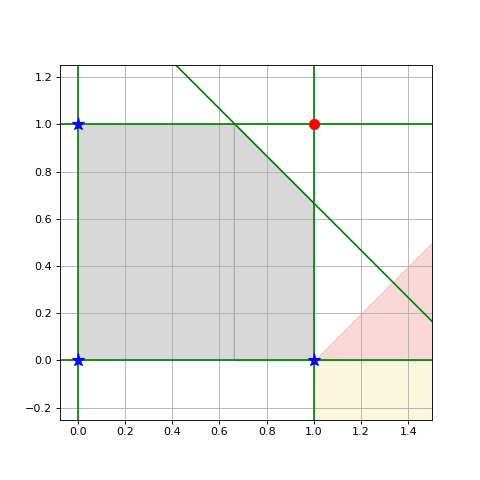}
\end{center}
    \caption{A numerical illustration to show why the approach of retrieving active constraints does not work in the Knapsack problems.}
    \vspace{2em}
    \label{fig:cave}
\end{figure}


\section{Experimental Details}
\label{appendix:expsetup}

\subsection{Hyperparameter }
We tuned all hyperparameters on the validation set and found that a learning rate of $0.005$ consistently yielded the lowest validation regret across all problems. Since the mapping between $\cost$ and $\feature$ follows the same data generation process in every case, we adopted this learning rate for all experiments.
Additionally, PFY includes an extra hyperparameter: the temperature parameter.
After tuning, based on validation regret, we set the following values of the temperature parameter.

\begin{table}[]
    \centering
    \begin{tabular}{cc}
    \toprule
  Problem & Temperature parameter value\\
  \toprule
        SP & 1. \\
        KP & 1.\\
        CFL & 1.\\
       \bottomrule
    \end{tabular}
    \caption{Value of Temperature Hyperparameter of the PFY Method}
    \label{tab:my_label}
\end{table}
\subsection{Implementation of DYS}
We adopt the implementation by \citet{mckenzie2024learning} to implement DYS-Net
\footnote{\url{https://github.com/mines-opt-ml/fpo-dys}}.
DYS-Net includes a few hyperparameters:
$\mu$, controls the strength of smoothing; scaling parameter $\alpha \in (0, 2/\mu)$; number of time \Eqref{eq:DYS_iteration} is iterated.   (For detailed explanations of these parameters, please refer to the original papers.)
In practice, we set the number of iterations to  100, because this results in good quality performances across all instances. We agree it might be possible to tune to specific problems to achieve even faster implementation without compromising performance.
Each iteration is implemented as a multi-layer perceptron (MLP), making the implementation computationally efficient.
We tune the parameter $\alpha$ on a validation set. In most cases, we keep it low $\approx 0.01$.
 Notably, DYS-Net implementation does not require pretraining, as DYS-Net contains no trainable parameters.
\subsection{DYS-Net Hyperparameter}
\begin{table}[]
    \centering
    \begin{tabular}{ccc}
    \toprule
  Problem & $\alpha$ & $\tau$\\
  \toprule
        SP-5 & 0.01 & 0. \\
        SP-10 & 0.01 & 1. \\
        SP-15 & 0.01 & 0. \\
        SP-25 & 0.01 & 0. \\
        KP-100 & 0.01 & 1.\\
        KP-200 & 0.01 & 0.\\
        KP-400 &  0.01 & 0.\\
        KP-500 &  0.01 & 0.\\
        CFL-50 & 0.01 & 0.1\\
        CFL-100 & 0.01 & 0.1\\
        CFL-200 & 0.01 & 0.1\\
        CFL-250 & 0.01 & 0.1\\
       \bottomrule
    \end{tabular}
    \vspace{1em}
    \caption{Value of $SCE^{DYS}$ Hyperparameters}
    \label{tab:my_label}
\end{table}



\newpage
\subsection{Tabular Result}
\begin{table}[]
    \centering
\begin{tabular}{llrr}
\toprule
 &  & Normalized Regret & Runtime (sec.) \\
Gridsize & Model &  &  \\
\midrule
\multirow[t]{12}{*}{5} & $Regret^{CVX}$ & 0.36 & 3.66 \\
 & $Regret^{DYS}$ & 0.35 & 0.93 \\
 & $SCE^{CVX}$ & 0.31 & 3.67 \\
 & $SCE^{DYS}$ & 0.32 & 0.94 \\
 & $SPO_{+}^{CVX}$ & 0.32 & 3.13 \\
 & $SPO_{+}^{DYS}$ & 0.33 & 0.92 \\
 & $SPO_{+}^{relax}$ & 0.32 & 0.73 \\
 & $SqDE^{CVX}$ & 0.36 & 3.71 \\
 & $SqDE^{DYS}$ & 0.34 & 0.73 \\
 & CaVE & 0.33 & 3.42 \\
 & MSE & 0.45 & 0.26 \\
 & PFY & 0.32 & 0.73 \\
\cline{1-4}
\multirow[t]{12}{*}{10} & $Regret^{CVX}$ & 0.45 & 3.25 \\
 & $Regret^{DYS}$ & 0.41 & 1.63 \\
 & $SCE^{CVX}$ & 0.35 & 3.24 \\
 & $SCE^{DYS}$ & 0.37 & 1.44 \\
 & $SPO_{+}^{CVX}$ & 0.37 & 3.48 \\
 & $SPO_{+}^{DYS}$ & 0.36 & 1.51 \\
 & $SPO_{+}^{relax}$ & 0.37 & 2.10 \\
 & $SqDE^{CVX}$ & 0.52 & 3.31 \\
 & $SqDE^{DYS}$ & 0.41 & 1.56 \\
 & CaVE & 0.48 & 11.89 \\
 & MSE & 0.46 & 0.26 \\
 & PFY & 0.37 & 2.12 \\
\cline{1-4}
\multirow[t]{9}{*}{15} & $Regret^{DYS}$ & 0.44 & 2.08 \\
 & $SCE^{DYS}$ & 0.40 & 2.10 \\
 & $SPO_{+}^{DYS}$ & 0.39 & 2.09 \\
 & $SPO_{+}^{p=5\%}$ & 0.44 & 2.94 \\
 & $SPO_{+}^{relax}$ & 0.40 & 6.15 \\
 & $SqDE^{DYS}$ & 0.47 & 2.04 \\
 & CaVE & 0.53 & 20.93 \\
 & MSE & 0.48 & 1.58 \\
 & PFY & 0.42 & 6.12 \\
\cline{1-4}
\multirow[t]{8}{*}{25} & $Regret^{DYS}$ & 0.48 & 3.85 \\
 & $SCE^{DYS}$ & 0.47 & 3.85 \\
 & $SPO_{+}^{DYS}$ & 0.45 & 3.77 \\
 & $SPO_{+}^{p=5\%}$ & 0.54 & 7.10 \\
 & $SPO_{+}^{relax}$ & 0.46 & 16.22 \\
 & $SqDE^{DYS}$ & 0.58 & 3.70 \\
 & MSE & 0.51 & 3.62 \\
 & PFY & 0.49 & 16.29 \\
\cline{1-4}
\end{tabular}    \caption{Result on Shortest Path Instances}
    \label{tab:sp}
\end{table}

\begin{table}[]
    \centering
\begin{tabular}{llrr}
\toprule
 &  & Normalized Regret & Runtime (sec.) \\
numitems & Model &  &  \\
\midrule
\multirow[t]{13}{*}{100} & $Regret^{CVX}$ & 0.30 & 6.00 \\
 & $Regret^{DYS}$ & 0.22 & 1.59 \\
 & $SCE^{CVX}$ & 0.18 & 5.97 \\
 & $SCE^{DYS}$ & 0.18 & 1.48 \\
 & $SPO_{+}^{CVX}$ & 0.19 & 5.78 \\
 & $SPO_{+}^{DYS}$ & 0.20 & 1.55 \\
 & $SPO_{+}^{ILP}$ & 0.20 & 2.29 \\
 & $SPO_{+}^{relax}$ & 0.20 & 1.31 \\
 & $SqDE^{CVX}$ & 0.24 & 5.96 \\
 & $SqDE^{DYS}$ & 0.18 & 1.55 \\
 & CaVE & 0.76 & 8.64 \\
 & MSE & 0.23 & 0.33 \\
 & PFY & 0.18 & 5.97 \\
\cline{1-4}
\multirow[t]{13}{*}{200} & $Regret^{CVX}$ & 0.32 & 6.38 \\
 & $Regret^{DYS}$ & 0.24 & 1.98 \\
 & $SCE^{CVX}$ & 0.19 & 6.41 \\
 & $SCE^{DYS}$ & 0.19 & 2.11 \\
 & $SPO_{+}^{CVX}$ & 0.19 & 6.42 \\
 & $SPO_{+}^{DYS}$ & 0.20 & 2.06 \\
 & $SPO_{+}^{ILP}$ & 0.20 & 4.98 \\
 & $SPO_{+}^{relax}$ & 0.20 & 2.28 \\
 & $SqDE^{CVX}$ & 0.31 & 6.40 \\
 & $SqDE^{DYS}$ & 0.20 & 1.93 \\
 & CaVE & 0.76 & 9.62 \\
 & MSE & 0.24 & 0.34 \\
 & PFY & 0.18 & 12.97 \\
\cline{1-4}
\multirow[t]{10}{*}{400} & $Regret^{DYS}$ & 0.25 & 2.65 \\
 & $SCE^{DYS}$ & 0.19 & 2.76 \\
 & $SPO_{+}^{DYS}$ & 0.21 & 2.72 \\
 & $SPO_{+}^{ILP}$ & 0.20 & 8.65 \\
 & $SPO_{+}^{p=5\%}$ & 0.20 & 1.56 \\
 & $SPO_{+}^{relax}$ & 0.20 & 4.63 \\
 & $SqDE^{DYS}$ & 0.20 & 2.46 \\
 & CaVE & 0.71 & 10.28 \\
 & MSE & 0.24 & 0.26 \\
 & PFY & 0.17 & 24.05 \\
\cline{1-4}
\multirow[t]{10}{*}{500} & $Regret^{DYS}$ & 0.24 & 3.85 \\
 & $SCE^{DYS}$ & 0.18 & 3.86 \\
 & $SPO_{+}^{DYS}$ & 0.20 & 3.81 \\
 & $SPO_{+}^{ILP}$ & 0.19 & 10.60 \\
 & $SPO_{+}^{p=5\%}$ & 0.20 & 2.33 \\
 & $SPO_{+}^{relax}$ & 0.19 & 5.60 \\
 & $SqDE^{DYS}$ & 0.20 & 3.87 \\
 & CaVE & 0.68 & 16.30 \\
 & MSE & 0.23 & 0.27 \\
 & PFY & 0.17 & 31.73 \\
\cline{1-4}
\end{tabular}    \caption{Result on Knapsack Instances}
    \label{tab:kp}
\end{table}

\begin{table}[]
    \centering
\begin{tabular}{lllrrr}
\toprule
&  &  & Normalized Regret & Runtime (sec.) \\
No. of Customers & No. of  Facilities & Model &  &  \\
\midrule
\multirow[t]{13}{*}{50} & \multirow[t]{13}{*}{5} & $Regret^{CVX}$ & 0.22 & 8.01 \\
 &  & $Regret^{DYS}$ & 0.09 & 1.65 \\
 &  & $SCE^{CVX}$ & 0.09 & 8.03 \\
 &  & $SCE^{DYS}$ & 0.08 & 1.72 \\
 &  & $SPO_{+}^{CVX}$ & 0.09 & 7.23 \\
 &  & $SPO_{+}^{DYS}$ & 0.09 & 1.58 \\
 &  & $SPO_{+}^{ILP}$ & 0.09 & 4.94 \\
 &  & $SPO_{+}^{relax}$ & 0.09 & 3.31 \\
 &  & $SqDE^{CVX}$ & 0.19 & 8.02 \\
 &  & $SqDE^{DYS}$ & 0.09 & 1.62 \\
 &  & CaVE & 0.11 & 12.19 \\
 &  & MSE & 0.11 & 0.34 \\
 &  & PFY & 0.08 & 4.98 \\
\cline{1-5} \cline{2-5}
\multirow[t]{13}{*}{100} & \multirow[t]{13}{*}{10} & $Regret^{CVX}$ & 0.20 & 20.48 \\
 &  & $Regret^{DYS}$ & 0.18 & 2.99 \\
 &  & $SCE^{CVX}$ & 0.09 & 20.88 \\
 &  & $SCE^{DYS}$ & 0.09 & 2.89 \\
 &  & $SPO_{+}^{CVX}$ & 0.12 & 21.85 \\
 &  & $SPO_{+}^{DYS}$ & 0.12 & 2.76 \\
 &  & $SPO_{+}^{ILP}$ & 0.12 & 15.47 \\
 &  & $SPO_{+}^{relax}$ & 0.12 & 11.53 \\
 &  & $SqDE^{CVX}$ & 0.11 & 19.38 \\
 &  & $SqDE^{DYS}$ & 0.12 & 2.79 \\
 &  & CaVE & 0.12 & 34.83 \\
 &  & MSE & 0.14 & 0.38 \\
 &  & PFY & 0.10 & 16.66 \\
\cline{1-5} \cline{2-5}
\multirow[t]{9}{*}{200} & \multirow[t]{9}{*}{10} & $Regret^{DYS}$ & 0.25 & 4.62 \\
 &  & $SCE^{DYS}$ & 0.10 & 4.55 \\
 &  & $SPO_{+}^{DYS}$ & 0.15 & 4.38 \\
 &  & $SPO_{+}^{ILP}$ & 0.15 & 23.53 \\
 &  & $SPO_{+}^{p=5\%}$ & 0.19 & 3.92 \\
 &  & $SPO_{+}^{relax}$ & 0.15 & 22.06 \\
 &  & $SqDE^{DYS}$ & 0.16 & 4.52 \\
 &  & MSE & 0.17 & 0.24 \\
 &  & PFY & 0.13 & 23.77 \\
\cline{1-5} \cline{2-5}
\multirow[t]{9}{*}{250} & \multirow[t]{9}{*}{10} & $Regret^{DYS}$ & 0.25 & 4.35 \\
 &  & $SCE^{DYS}$ & 0.10 & 4.25 \\
 &  & $SPO_{+}^{DYS}$ & 0.15 & 4.08 \\
 &  & $SPO_{+}^{ILP}$ & 0.15 & 28.85 \\
 &  & $SPO_{+}^{p=5\%}$ & 0.19 & 4.36 \\
 &  & $SPO_{+}^{relax}$ & 0.15 & 27.70 \\
 &  & $SqDE^{DYS}$ & 0.16 & 4.24 \\
 &  & MSE & 0.17 & 0.27 \\
 &  & PFY & 0.13 & 29.24 \\
\cline{1-5} \cline{2-5}
\end{tabular}    \caption{Result on Capacitated Facility Location Problem Instances}
    \label{tab:cfl}
\end{table}

\end{document}